\documentclass{article}
\usepackage[final]{neurips_2024}
\usepackage[utf8]{inputenc} 
\usepackage[T1]{fontenc}    
\usepackage{hyperref}       
\usepackage{url}            
\usepackage{booktabs}       
\usepackage{amsfonts}       
\usepackage{nicefrac}       
\usepackage{microtype}      
\usepackage{xcolor}         
\usepackage{multicol}
\usepackage{multirow}
\usepackage{amsmath}
\usepackage{amssymb}
\usepackage{mathtools}
\usepackage{amsthm}
\usepackage{wrapfig}
\usepackage[whole]{bxcjkjatype}
\usepackage{bm}
\usepackage[capitalize,noabbrev]{cleveref}
\theoremstyle{plain}
\newtheorem{theorem}{Theorem}[section]

\newtheorem{lemma}[theorem]{Lemma}
\newtheorem{corollary}[theorem]{Corollary}
\theoremstyle{definition}

\newtheorem{assumption}[theorem]{Assumption}
\theoremstyle{remark}

\newcommand{\B}[1]{{\bm #1}}
\newcommand{\mac}[1]{{\mathcal #1}}
\newcommand{\mab}[1]{{\mathbb #1}}

\usepackage{stackengine}
\def\delequal{\mathrel{\ensurestackMath{\stackon[1pt]{=}{\scriptstyle\Delta}}}}

\makeatletter
\newcommand{\figcaption}[1]{\def\@captype{figure}\caption{#1}}
\newcommand{\tblcaption}[1]{\def\@captype{table}\caption{#1}}
\makeatother

\title{Controlling Continuous Relaxation for \\
Combinatorial Optimization}

\author{Yuma Ichikawa \\
  Fujitsu Limited, Kanagawa, Japan\\ 
  Department of Basic Science, University of Tokyo
  }

\begin{document}

\maketitle

\begin{abstract}
    Unsupervised learning (UL)-based solvers for combinatorial optimization (CO) train a neural network that generates a soft solution by directly optimizing the CO objective using a continuous relaxation strategy.
    These solvers offer several advantages over traditional methods and other learning-based methods, particularly for large-scale CO problems.
    However, UL-based solvers face two practical issues: (I) an optimization issue, where UL-based solvers are easily trapped at local optima, and (II) a rounding issue, where UL-based solvers require artificial post-learning rounding from the continuous space back to the original discrete space, undermining the robustness of the results.
    This study proposes a \underline{C}ontinuous \underline{R}elaxation \underline{A}nnealing (\textbf{CRA}) strategy, an effective rounding-free learning method for UL-based solvers. 
    CRA introduces a penalty term that dynamically shifts from prioritizing continuous solutions, effectively smoothing the non-convexity of the objective function, to enforcing discreteness, eliminating artificial rounding.
    Experimental results demonstrate that CRA significantly enhances the performance of UL-based solvers, outperforming existing UL-based solvers and greedy algorithms in complex CO problems. Additionally, CRA effectively eliminates artificial rounding and accelerates the learning process.
\end{abstract}

\section{Introduction}\label{sec:introduction}
The objective of combinatorial optimization (CO) problems is to find the optimal solution from a discrete space, and these problems are fundamental in many real-world applications \citep{papadimitriou1998combinatorial}.
Most CO problems are \textsf{NP}-hard or \textsf{NP}-complete; making it challenging to solve large-scale problems within feasible computational time.
Traditional methods frequently depend on heuristics to find approximate solutions; however, considerable insights into the specific problems are required. 
Alternatively, problems can be formulated as integer linear programming (ILP) and solved using ILP solvers. However, ILP lacks scalability when applied to problems with graph structures.

Recently, several studies have used machine learning methods to handle CO problems by learning heuristics.
Most of these studies focus on supervised learning (SL)-based solvers \citep{hudson2021graph, joshi2019efficient, gasse2019exact, selsam2018learning, khalil2016learning}, which require optimal solutions to CO problems as supervision during training.
However, obtaining optimal solutions is challenging in practice, and SL-based solvers often fail to generalize well \citep{yehuda2020s}.
Reinforcement learning (RL)-based solvers \citep{ yao2019experimental, chen2019learning, yolcu2019learning, nazari2018reinforcement, khalil2017learning, bello2016neural} avoid the need for optimal solutions but often suffer from notoriously unstable training due to poor gradient estimation and hard explorations \citep{mnih2015human, tang2017exploration, espeholt2018impala}.
Unsupervised learning (UL)-based solvers have recently attracted much attention \citep{schuetz2022combinatorial, karalias2020erdos, amizadeh2018learning}.
UL-based solvers follow a continuous relaxation approach, training a UL model to output a \textit{soft solution} to the relaxed CO problem by directly optimizing a differentiable objective function, offering significantly stable and fast training even for large-scale CO problems.
Notably, the physics-inspired GNN (PI-GNN) solver \citep{schuetz2022combinatorial} employs graph neural networks (GNN) to automatically learn instance-specific heuristics and performs on par with or outperforms existing solvers for CO problems with millions of variables without optimal solutions. 
While these offer some advantages over traditional and other machine learning-based solvers, they face two practical issues.
The first issue is an optimization issues where UL-based solvers are easily trapped at local optima.
Due to this issue, \citet{angelini2023modern} demonstrated that the PI-GNN solver \citep{schuetz2022combinatorial} could not achieve results comparable to those of the degree-based greedy algorithm (DGA) \citep{angelini2019monte} on maximum independent set (MIS) problems in random regular graphs (RRG). 
\citet{wang2023unsupervised} also pointed out the importance of using dataset or history, and initializing the GNN with outputs from greedy solvers to help the PI-GNN solver overcome optimization challenges. 
This issue is a crucial bottleneck to the applicability of this method across various real-world applications.
The second issue relates to the inherent ambiguity of the continuous relaxation approach. 
This approach necessitates artificial rounding from the soft solution, which may include continuous values, back to the original discrete solution, potentially undermining the robustness of the results. 
While linear relaxation can provide an optimal solution for original discrete problems on bipartite graphs \citep{hoffman2010integral}, it typically leads to solutions with $1/2$ values, which is known to half-integrality \citep{nemhauser1974properties}, 
in which existing rounding methods \citep{schuetz2022graph, wang2022unsupervised} completely lose their robustness.
For NP-hard problems with graph structures, such as the MIS and MaxCut, semidefinite programming (SDP) relaxations have been proposed as effective approximation methods \citep{lovasz1979shannon, goemans1995improved}. 
However, these approaches rely on rounding techniques, such as spectral clustering \citep{von2007tutorial}, to transform relaxed solutions into feasible ones, which often fails to obtain optimal solutions.

To address these issues, we propose the \underline{C}ontinuous \underline{R}elaxation \underline{A}nnealing (\textbf{CRA}). 
CRA introduces a penalty term to control the continuity and discreteness of the relaxed variables, with a parameter $\gamma$ to regulate the intensity of this penalty term. 
When the parameter $\gamma$ is small, the relaxed variable tends to favor continuous solutions, whereas a large $\gamma$ biases them toward discrete values. 
This penalty term also effectively eliminates local optimum.
Moreover, a small $\gamma$ forces the loss function to approach a simple convex function, encouraging active exploration within the continuous space.
CRA also includes an annealing process, where $\gamma$ is gradually increased until the relaxed variables approach discrete values, eliminating the artificial rounding from the continuous to the original discrete space after learning. 
In this study, the solver that applies the CRA to the PI-GNN solver is referred to as the CRA-PI-GNN solver.
We also demonstrate the benefits of the CRA through experiments on benchmark CO problems, including MIS, maximum cut (MaxCut), and diverse bipartite matching (DBM) problems across graphs of varying sizes and degrees.
The experimental results show that the CRA significantly enhances the performance of the PI-GNN solver, outperforming the original PI-GNN solver, other state-of-the-art learning-based baselines, and greedy algorithms.
This improvement is achieved by directly optimizing each instance without any history, e.g., previous optimal solutions and the information of other instances.
Additionally, these experiments indicate that the CRA accelerates the learning process of the PI-GNN solver.
Notably, these results overcome the limitations pointed out by \citet{angelini2023modern, wang2023unsupervised}, highlighting the further potential of UL-based solvers.

\paragraph{Notation}
We use the shorthand expression $[N]=\{1, 2, \ldots, N\}$, where $N \in \mab{N}$. $I_{N} \in \mab{R}^{N \times N}$ denotes an $N\times N$ identity matrix, $\B{1}_{N}$ denotes the vector $(1, \ldots, 1)^{\top} \in \mab{R}^{N}$, and $\B{0}_{N}$ denotes the vector $(0, \ldots, 0)^{\top} \in \mab{R}^{N}$. 
$G(V, E)$ represents an undirected graph, where $V$ is the set of nodes with cardinality $|V| = N$, and $E \subseteq V \times V$ denotes the set of edges. 
For a graph $G(V, E)$, $A_{ij}$ denotes the adjacency matrix, where $A_{ij}=0$ if an edge $(i, j)$ does not exist and $A_{ij} > 0$ if the edge is present.

\section{Background}\label{sec:background}
\paragraph{Combinatorial optimization} 
The goal of this study is to solve the following CO problem.
\begin{equation}
    \min_{\B{x} \in \{0, 1\}^{N}} f(\B{x}; C)~~~\mathrm{s.t.}~~~\B{x} \in \mac{X}(C),
\end{equation}
where $C \in \mac{C}$ denotes instance-specific parameters, such as a graph $G=(V, E)$, and $\mac{C}$ represents the set of all possible instances. 
$f: \mac{X} \times \mac{C} \to \mab{R}$ denotes the cost function. 
Additionally, $\B{x} = (x_{i})_{1\le i \le N} \in \{0, 1\}^{N}$ is a binary vector to be optimized, and $\mac{X}(C) \subseteq \{0, 1\}^{N}$ denotes the feasible solution space, typically defined by the following equality and inequality constraints.
\begin{equation*}
    \mac{X}(C) = \{\B{x} \in \{0, 1\}^{N} \mid \forall i \in [I],~g_{i}(\B{x};C) \le 0,~\forall j \in [J],~ h_{j}(\B{x};C) = 0\}, ~~I, J \in \mab{N}, 
\end{equation*}
Here, for $i \in [I]$, $g_{i}: \{0,1\}^{N} \times \mac{C} \to \mab{R}$ denotes the inequality constraint, and for $j \in [J]$, $h_{j}: \{0, 1\}^{N} \times \mac{C} \to \mab{R}$ denotes the equality constraint. Following UL-based solvers \citep{wang2022unsupervised, schuetz2022combinatorial, karalias2020erdos}, we reformulate the constrained problem into an equivalent unconstrained form using 
the penalty method \citep{smith1997penalty}:
\begin{equation}
    \min_{\B{x}} l(\B{x}; C, \B{\lambda}),~~l(\B{x}; C, \B{\lambda}) \delequal f(\B{x}; C) + \sum_{i=1}^{I+J} \lambda_{i} v_{i}(\B{x}; C).
\end{equation}
Here, for all $i \in [I+J]$, $v: \{0, 1\}^{N} \times \mac{C} \to \mab{R}$ is the penalty term, which increases when the constraints are violated. 
For example, the penalty term is defined as follows:
\begin{align*}
    &\forall i \in [I], j \in [J],~v_{i}(\B{x};C) = \max (0, g_{i}(\B{x};C)),~~\forall j \in [J],~v_{j}(\B{x};C) = (h_{j}(\B{x};C))^{2}
\end{align*}
and $\B{\lambda} = (\lambda_{i})_{1\le i \le I+J} \in \mab{R}^{I+J}$ denotes the penalty parameters that control the trade-off between constraint satisfaction and cost optimization. 
Note that, as $\B{\lambda}$ increases, the penalty for constraint violations becomes more significant. 
In the following, we provide an example of this formulation.  

\paragraph{Example: MIS problem}
The MIS problem is a fundamental NP-hard problem \citep{karp2010reducibility}, defined as follows.
Given an undirected graph $G(V, E)$,  an independent set (IS) is a subset of nodes $\mac{I} \in V$ where any two nodes are not adjacent.
The MIS problem aims to find the largest IS, denoted as $\mac{I}^{\ast}$. 
In this study, $\rho$ denotes the IS density, defined as $\rho = |\mac{I}|/|V|$.
Following \citet{schuetz2022combinatorial}, a binary variable $x_{i}$ is assigned to each node $i \in V$. The MIS problem can be formulated as follows: 
\begin{equation}
    \label{eq:mis-qubo}
    f(\B{x}; G, \lambda) = - \sum_{i \in V} x_{i} + \lambda \sum_{(i,j) \in E} x_{i} x_{j},
\end{equation}
where the first term maximizes the number of nodes assigned a value of $1$, and the second term penalizes adjacent nodes assigned $1$ according to the penalty parameter $\lambda$.

\subsection{Unsupervised learning based solvers}\label{sec:background-UL-based-solvers}
Learning for CO problems involves training an algorithm $\mac{A}_{\B{\theta}}(\cdot): \mac{C} \to \{0, 1\}^{N}$ parameterized by a neural network (NN), where $\B{\theta}$ denotes the parameters. 
For a given instance $C \in \mac{C}$, this algorithm generates a valid solution $\hat{\B{x}} = \mac{A}_{\theta}(C) \in \mac{X}(C)$ and aims to minimize $f(\hat{\B{x}}; C)$.
Several approaches have been proposed to train $\mac{A}_{\theta}$. 
This study focuses on UL-based solvers, which do not use a labeled solution $\B{x}^{\ast} \in \mathrm{argmin}_{\B{x} \in \mac{X}(C)} f(\B{x}; C)$ during training \citep{wang2022unsupervised, schuetz2022combinatorial, karalias2020erdos, amizadeh2018learning}. 
In the following, we outline the details of the UL-based solvers.

The UL-based solvers employ a continuous relaxation strategy to train NN.
This continuous relaxation strategy reformulates a CO problem into a continuous optimization problem by converting discrete variables into continuous ones. 
A typical example of continuous relaxation is expressed as follows:
\begin{equation*}
    \min_{\B{p}} \hat{l}(\B{p}; C, \B{\lambda}),~~\hat{l}(\B{p}; C, \B{\lambda}) \delequal \hat{f}(\B{p}; C) + \sum_{i=1}^{m+p} \lambda_{i} \hat{v}_{i}(\B{p}; C), 
\end{equation*}
where $\B{p} = (p_{i})_{1\le i \le N} \in [0, 1]^{N}$ represents a set of relaxed continuous variables, where each binary variable $x_{i} \in \{0, 1\}$ is relaxed to a continuous counterpart $p_{i} \in [0, 1]$, and $\hat{f}: [0, 1]^{N} \times \mac{C} \to \mab{R}$ denotes the relaxed form of $f$ such that $\hat{f}(\B{x}; C) = f(\B{x}; C)$ for $\B{x} \in \{0, 1\}^{N}$. The relation between each constraint $v_{i}$ and its relaxation $\hat{v}_{i}$ is similar for $i \in [I+J]$, meaning that $\forall i \in [I+J],~\hat{v}_{i}(\B{x}; C) = v_{i}(\B{x}; C)$ for $\B{x} \in \{0, 1\}^{N}$. \citet{wang2022unsupervised} and \citet{schuetz2022combinatorial} formulated $\mac{A}_{\theta}(C)$ as the relaxed continuous variables, defined as $\mac{A}_{\theta}(\cdot): \mac{C} \to [0, 1]^{n}$. 
In the following discussions, we denote $\mac{A}_{\theta}$ as $\B{p}_{\theta}$ to make the parametrization of the relaxed variables explicit. 
Then, $\B{p}_{\theta}$ is optimized by directly minimizing the following label-independent function:
\begin{equation}
    \label{eq:gnn-loss}
    \hat{l}(\B{\theta}; C, \B{\lambda}) \delequal \hat{f}(\B{p}_{\theta}(C); C) + \sum_{i=1}^{I+J} \lambda_{i} \hat{v}_{i}(\B{p}_{\theta}(C); C).
\end{equation}
After training, the relaxed solution $\B{p}_{\B{\theta}}$ is converted into discrete variables using artificial rounding  $\B{p}_{\B{\theta}}$, where $\forall i \in [N],~x_{i} = \mathrm{int}(p_{\B{\theta}, i}(C))$ based on a threshold \citep{schuetz2022combinatorial}, or alternatively, a greedy method \citep{wang2022unsupervised}.
Two types of schemes for UL-based solvers have been developed based on this formulation.

\paragraph{(Type I) Learning generalized heuristics from history/data}
One approach, proposed by \citet{karalias2020erdos}, aims to automatically learn effective heuristics from historical dataset instances $\mac{D}=\{C_{\mu}\}_{\mu=1}^{P}$ and then apply these learned heuristics to a new instance $C^{\ast}$, through inference.
Note that this method assumes that either the training dataset is easily obtainable or that meaningful data augmentation is feasible.
Specifically, given a set of training instances $\mac{D}=(C_{\mu})$, sampled independently and identically from a distribution $P(C)$, the goal is to minimize the average loss function $\min_{\B{\theta}} \sum_{\mu=1}^{P} l(\B{\theta}; C_{\mu}, \B{\lambda})$.
However, this method does not guarantee quality for a test instance, $C^{\ast}$. 
Even if the training instances $\mac{D}$ are extensive and the test instance $C$ follows $P(C)$, low average performance $\mab{E}_{C \sim P(C)}[\hat{l}(\theta; C)]$ may not guarantee a low $l(\theta; C)$ for on a specific $C$. 
To address this issue, \citet{wang2023unsupervised} introduced a meta-learning approach where NNs aim to provide good initialization for future instances rather than direct solutions.

\paragraph{(Type II) Learning effective heuristics on a specific single instance}
Another approach, known as the PI-GNN solver \citep{schuetz2022combinatorial, schuetz2022graph}, automatically learns instance-specific heuristics for a single instance using the instance parameter $C$ by directly applying Eq. \eqref{eq:gnn-loss}. This approach addresses CO problems on graphs, where $C=G(V, E)$, and employs GNNs for the relaxed variables $p_{\theta}(G)$.
Here, an $L$-layered GNN is trained to directly minimize $\hat{l}(\B{\theta}; C, \B{\lambda})$, taking as input a graph $G$ and the embedding vectors on its nodes, and outputting the relaxed solution $\B{p}_{\B{\theta}}(G) \in [0, 1]^{N}$. 
A detailed description of GNNs is provided in Appendix \ref{subsec:graph-neural-network}.
Note that this setting is applicable even when the training dataset $\mac{D}$ is difficult to obtain.
The overparameterization of relaxed variables is expected to smooth the objective function by introducing additional parameters to the optimization problem, similar to the kernel method.
However, minimizing Eq.~\ref{eq:gnn-loss} for a single instance can be time-consuming compared to the inference process. Nonetheless, for large-scale CO problems, this approach has been reported to outperform other solvers in terms of both computational time and solution quality \citep{schuetz2022combinatorial, schuetz2022graph}.

Note that, while both UL-based solvers for multiple instances (Type I) and individual instances (Type II) are valuable, this study focuses on advancing the latter: a UL-based solver for a single instance.
Both types of solvers are applicable to cost functions that meet a particular requirement due to their reliance on a gradient-based algorithm to minimize Eq~\eqref{eq:gnn-loss}.
\begin{assumption}[Differentiable cost function]
The relaxed loss function $\hat{l}(\B{\theta}; C, \B{\lambda})$ and its partial derivative $\nicefrac{\partial \hat{l}(\B{\theta}; C, \B{\lambda})}{\partial \B{\theta}}$ are accessible during the optimization process.
\end{assumption}
These requirements encompass a nonlinear cost function and interactions involving many-body interactions, extending beyond simple two-body interactions.

\section{Continuous relaxation annealing for UL-based solvers}
In this section, we discuss the practical issues associated with UL-based solvers and then introduce continuous relaxation annealing (CRA) as a proposed solution.

\subsection{Motivation: practical issues of UL-based solvers}\label{sec:how-can-it-fail}
UL-based solvers (Type II) \citep{schuetz2022combinatorial, schuetz2022graph} are effective in addressing large-scale CO problems. 
However, these solvers present following two practical issues, highlighted in several recent studies \citep{wang2023unsupervised, angelini2023modern}. 
Additionally, we numerically validate these issues; see Appendix \ref{subsec:app-how-can-it-fail} for detailed results.

\paragraph{(I) Ambiguity in rounding method after learning}
UL-based solvers employ a continuous relaxation strategy to train NNs and then convert the relaxed continuous variables into discrete binary values through artificial rounding as discussed in Section \ref{sec:background-UL-based-solvers}.
This inherent ambiguity in continuous relaxation strategy often results in potential discrepancies between the optimal solutions of the original discrete CO problem and those of the relaxed continuous one.
Continuous relaxation expands the solution space, often producing continuous values that lower the cost compared to an optimal binary value.
Indeed, while linear relaxation can provide an optimal solution for discrete problems on bipartite graphs \citep{hoffman2010integral}, it typically results in solutions with $\nicefrac{1}{2}$ values, which is known to half-integrality \citep{nemhauser1974properties}.
Existing rounding methods \citep{schuetz2022graph, wang2022unsupervised} often lose robustness in these scenarios.
In practice, PI-GNN solver often outputs values near $\nicefrac{1}{2}$, underscoring the limitations of current rounding techniques for UL-based solvers.

\paragraph{(II) Difficulty in optimizing NNs}
Recently, \citet{angelini2023modern} demonstrated that PI-GNN solver falls short of achieving results comparable to those of the degree-based greedy algorithm (DGA) \citep{angelini2019monte} when solving the MIS problems on RRGs. 
\citet{angelini2023modern} further emphasized the importance of evaluating UL-based solvers on complex CO problems, where greedy algorithms typically perform worse.
A representative example is the MIS problems on RRGs with a constant degree $d>16$, where a clustering transition in the solution space creates barriers that impede optimization. 
Moreover, \citet{wang2023unsupervised} emphasized the importance of using training/historical datasets, $\mac{D}=\{C_{\mu}\}_{1\le \mu \le P}$, which contain various graphs and initialization using outputs from greedy solvers, such as DGA and RGA for MIS problems.
Their numerical analysis indicated that PI-GNN solver tends to get trapped in local optima when directly optimized directly for a single instance without leveraging a training dataset $\mac{D}$. 
However, in a practical setting, systematic methods for generating or collecting training datasets $\mac{D}$ to effectively avoid local optima remains unclear.  
Additionally, training on instances that do not contribute to escaping local optima is time-consuming.
Therefore, it is crucial to develop an effective UL-based solver that can operate on a single instance without relying on training data, $\mac{D}$.
Our numerical experiments, detailed in Appendix \ref{subsec:app-how-can-it-fail}, also confirmed this optimization issue. They demonstrated that as problem complexity increases, the PI-GNN solver is often drawn into trivial local optima, $\B{p}_{\B{\theta}}=\B{0}_{N}$, in certain problems. This entrapment results in prolonged plateaus that significantly slow down the learning process and, in especially challenging cases, can render learning entirely infeasible.
Our numerical experiments, detailed in Appendix \ref{subsec:app-how-can-it-fail}, also validated this optimization issue, demonstrating that as the problem complexity increases, PI-GNN solver tends to be absorbed into the trivial local optima $\B{p}_{\B{\theta}}=\B{0}_{N}$ in some problems, resulting in prolonged plateaus which significantly decelerates the learning process and, in particularly challenging cases, can render learning entirely infeasible.

\subsection{Continuous relaxation annealing}\label{subsec:continuou-relaxation-annealing}
\paragraph{Penalty term to control discreteness and continuity}
To address these issues, we propose a penalty term to control the balance between discreteness and continuity in the relaxed variables, formulated as follows:
\begin{equation}
    \label{eq:penalty-cost}
    \hat{r}(\B{p}; C, \B{\lambda}, \gamma) = \hat{l}(\B{p}; C, \B{\lambda}) + \gamma \Phi(\B{p}),~~
    \Phi(\B{p}) \delequal \sum_{i=1}^{N} (1-(2p_{i}-1)^{\alpha}),~~\alpha \in \{2n \mid n \in \mab{N}_{+}\},
\end{equation}
where $\gamma \in \mab{R}$ is a penalty parameter, and the even number $\alpha$ denote a curve rate.
When $\gamma$ is negative, i.e., $\gamma < 0$, the relaxed variables tend to favor the continuous space, smoothing the non-convexity of the objective function $\hat{l}(\B{p}; C, \B{\lambda})$ due to the convexity of the penalty term $\Phi(\B{p})$. 
In contrast, when $\gamma$ is positive, i.e., $\gamma > 0$, the relaxed variables tend to favor discrete space, smoothing out the continuous solution into discrete solution.
Formally, the following theorem holds as $\lambda$ approaches $\pm \infty$.
\begin{theorem}
    \label{theorem:gamma-annealing-limit}
    Assuming the objective function $\hat{l}(\B{p};C)$ is bounded within the domain $[0, 1]^{N}$, as $\gamma \to +\infty$, the relaxed solutions $\B{p}^{\ast} \in \mathrm{argmin}_{\B{p}} \hat{r}(\B{p}; C, \B{\lambda}, \gamma)$ converge to the original solutions $\B{x}^{\ast} \in \mathrm{argmin}_{\B{x}} l(\B{x}; C, \B{\lambda})$. Moreover, as $\gamma \to -\infty$, the loss function $\hat{r}(\B{p}; C, \B{\lambda}, \gamma)$ becomes convex, and the relaxed solution $\nicefrac{\B{1}_{N}}{2} = \mathrm{argmin}_{\B{p}} \hat{r}(\B{p}, C, \B{\lambda}, \gamma)$ is unique.
\end{theorem}
For the detailed proof, refer to Appendix \ref{sec:proof-theorem}. 
Theorem \ref{theorem:gamma-annealing-limit} can be generalized for any convex function $\Phi(\B{p}; C)$ that has a unique maximum at  $\nicefrac{\B{1}_{N}}{2}$ and achieves a global minimum for all $\B{p} \in \{0, 1\}^{N}$; an example is binary cross entropy $\Phi_{\mathrm{CE}}(\B{p}) = \sum_{i=1}^{N} (p_{i} \log p_{i} + (1-p_{i})\log (1-p_{i}))$, introduced by \citet{sun2022annealed, sanokowski2024variational} for the UL-based solvers (Type I).
Additionally, the penalty term eliminates the stationary point $\B{p}^{\ast} = \B{0}_{N}$ described in Section \ref{sec:how-can-it-fail}, preventing convergence to a plateau.
For UL-based solvers, the penalty term is expressed as follows: 
\begin{equation}
\label{eq:penalty-cost-gnn}
    \hat{r}(\B{\theta}; C, \B{\lambda}, \gamma) = \hat{l}(\B{\theta}; C, \B{\lambda}) + \gamma \Phi(\B{\theta};C),
\end{equation}
where $\Phi(\B{\theta}; C) \delequal \Phi(\B{p}_{\theta}(C))$.
According to Theorem~\ref{theorem:gamma-annealing-limit}, setting a sufficiently large $\gamma$ value cases the relaxed variables to approach nearly discrete values.
We can also generalize this penalty term $\Phi(\B{\theta}; C)$, to Potts variables optimization, including coloring problems \citep{schuetz2022graph}, and mixed-integer optimization; refer to Appendix \ref{subsec:app-generalize-potts}.

\paragraph{Annealing penalty term}
\begin{wrapfigure}{r}{0.42\textwidth} 
    \vspace{-20pt}
  \centering
  \includegraphics[width=0.41\textwidth]{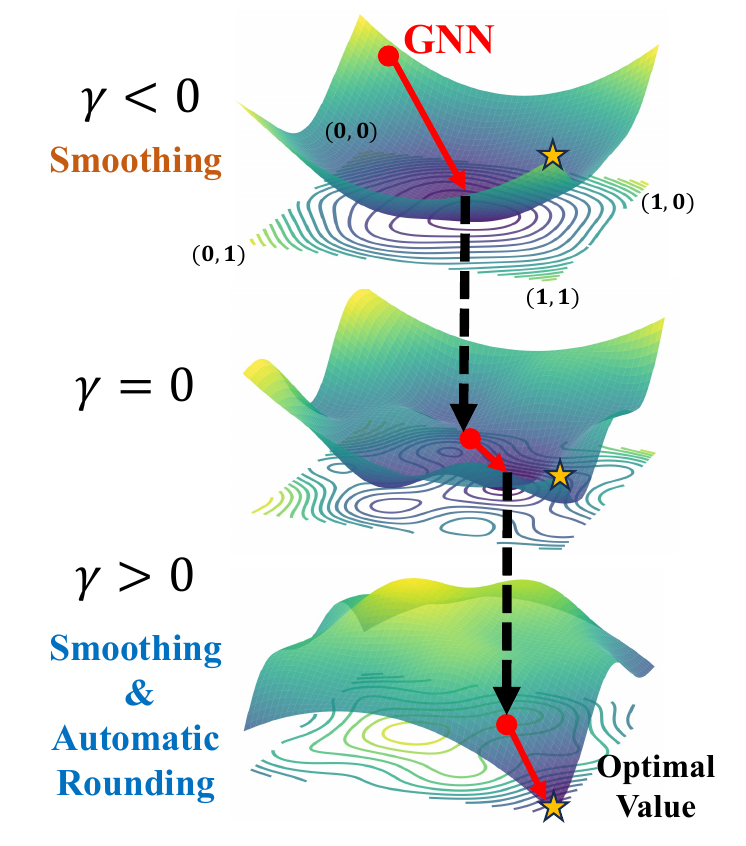}
  \caption{Annealing strategy. 
  When $\gamma < 0$, it facilitates exploration by reducing the non-convexity of the objective function. As $\gamma$ increases, it promotes optimal discrete solutions by smoothing away suboptimal continuous ones.}
  \label{fig:image-annealing-rounding}
  \vspace{-5pt}
\end{wrapfigure}
We propose an annealing strategy that gradually anneals the penalty parameter $\gamma$ in Eq.~\eqref{eq:penalty-cost-gnn}. 
Initially, a negative gamma value, i.e., $\gamma < 0$, is chosen to leverage the properties, facilitating broad exploration  by smoothing the non-convexity of $\hat{l}(\B{\theta}; C, \B{\lambda})$ and eliminating the stationary point $\B{p}^{\ast}=\B{0}_{N}$ to avoid the plateau, as discussed in Section \ref{sec:how-can-it-fail}.
Subsequently, the penalty parameter $\gamma$ is gradually increased to a positive value, $\gamma > 0$, with each update of the trainable parameters (one epoch), until the penalty term approaches zero, i.e., $\Phi(\B{\theta}, C) \approx 0$, to automatically round the relaxed variables by smoothing out suboptimal continuous solutions oscillating between $1$ or $0$. 
A conceptual diagram of this annealing process is shown in Fig.~\ref{fig:image-annealing-rounding}.

Note that employing the binary cross-entropy 
$\Phi_{\mathrm{CE}}(\B{p})$ is infeasible for UL-based solvers when $\gamma > 0$, as the gradient $\nicefrac{\partial \Phi_{\mathrm{CE}}(\B{p})}{\partial p_{i}}$ diverges to $\pm \infty$ at $0$ or $1$. 
In deed, when $\gamma=0$, most relaxed variables typically approach binary values, with a relatively small number of variables oscillating between $0$ and $1$. 
This gradient divergence issue in  $\Phi_{\mathrm{CE}}(\B{p})$ makes the learning infeasible without additional techniques, such as gradient clipping.
In contrast, the gradient of the penalty term in Eq.~\ref{eq:penalty-cost}, $\nicefrac{\partial \Phi(\B{p})}{\partial p_{i}}$, is bounded within $[-2\alpha, 2\alpha]$ for any $\gamma$, preventing the gradient divergence issue seen in  $\Phi_{\mathrm{CE}}(\B{p})$. 
Additionally, by increasing $\alpha$, the absolute value of the gradient near $\nicefrac{1}{2}$ becomes smaller, allowing for control over the smoothing strength toward a discrete solution near $\nicefrac{1}{2}$.

We also propose an early stopping strategy that monitors both the loss function and the penalty term, halting the annealing and learning processes when the penalty term approaches zero, i.e., $\Phi(\B{\theta}; C) \approx 0$. 
Various annealing schedules can be considered; in this study, we employ the following scheduling: $\gamma(\tau+1) \leftarrow \gamma(\tau) + \varepsilon$, where the scheduling rate $\varepsilon \in \mab{R}_{+}$ is a small constant, and $\tau$ denotes the update iterations of the trainable parameters. 
We refer to the PI-GNN solver with this continuous relaxation annealing as CRA-PI-GNN solver. 
Here, two additional hyperparameters are introduced: the initial scheduling value $\gamma(0)$ and the scheduling rate $\varepsilon$. Numerical experiments suggest that better solutions are obtained when $\gamma(0)$ is set to a small negative value and $\varepsilon$ is kept low. 
The ablation study are presented in Appendix \ref{subsec:ablation-study}.

\section{Related Work}\label{sec:related-work}
    Previous works on UL-based solvers have addressed various problems, such as MaxCut problems \citep{yao2019experimental} and traveling salesman problems \citep{hudson2021graph}, using carefully tailored problem-specific objectives.
    Some studies have also explored constraint satisfaction problems \citep{amizadeh2018learning, toenshoff2019run}, but applying these approaches to broader CO problems often requires problem-specific reductions.   \citet{karalias2020erdos} proposed Erd\H{o}s Goes Neural (EGN) solver, 
    an UL-based solver for general CO problems based on Erd\H{o}s' probabilistic method.
    This solver generate solutions through an inference process using training instances.
    Subsequently, \citet{wang2022unsupervised} proposed an entry-wise concave continuous relaxation, broadening the EGN solver to a wide range of CO problems.
    In contrast, \citet{schuetz2022combinatorial, schuetz2022graph} proposed PI-GNN solver, an UL-based solver for a single CO problems that automatically learns problem-specific heuristics during the training process. 
    However, \citet{angelini2023modern, boettcher2023inability} pointed out the optimization difficulties where PI-GNN solver failed to achieve results comparable to those of greedy algorithms. 
    \citet{wang2023unsupervised} also claimed optimization issues with PI-GNN solver, emphasizing the importance of learning from training data and history to overcome local optima.
    They then proposed Meta-EGN solvers, a meta-learning approach that updates NN network parameters for individual CO problem instances.
    Furthermore, to address these optimization issue, \citet{lin2023continuation, sun2022annealed, sanokowski2024variational} proposed annealing strategy similar to simulated annealing \citep{kirkpatrick1983optimization}.

\section{Experiments}\label{sec:experiment}
We begin by evaluating the performance of CRA-PI-GNN solver on the MIS and the MaxCut benchmark problems across multiple graphs of varying sizes, demonstrating that CRA effectively overcomes optimization challenges without relying on data/history $\mac{D}$. 
We then extend the evaluation to the DBM problems, showing the applicability to more practical CO problems. 
For the objective functions and the detailed explanations, refer to Appendix \ref{sec:theoretical-background-co-problem}.

\subsection{Experimental settings}\label{sub-sec:config}
\paragraph{Baseline methods}
In all experiments, the baseline methods include the PI-GNN solver \citep{schuetz2022combinatorial} as the direct baseline of a UL-based solver for a single instance.
For the MIS problems, we also consider the random greedy algorithm (RGA) and DGA \citep{angelini2019monte} as heuristic baselines. 
For the MaxCut problems, RUN-CSP solver \citep{toenshoff2019run} is considered as an additional baseline, and a standard greedy algorithm and SDP based approximation algorithm \citep{goemans1995improved} are considered as an additional classical baseline. 
The parameters for the Goemans-Williamson (GW) approximation are all set according to the settings in \citet{schuetz2022graph}. The implementation used the open-source CVXOPT solver with CVXPY \citep{cvx_graph_algorithms} as the modeling interface. 
Note that we do not consider UL-based solvers for learning generalized heuristics \citep{karalias2020erdos, wang2022unsupervised, wang2023unsupervised}, which rely on training instances $\mac{D}=\{C_{\mu}\}_{\mu=1}^{P}$.
The primary objective of this study is to evaluate whether CRA-PI-GNN solver can surpass the performance of both PI-GNN solver and greedy algorithms.
However, for the MIS problem, EGN solver \citep{karalias2020erdos} and Meta-EGN solver \citep{wang2023unsupervised} are considered to confirm that CRA can overcome the optimization issues without training instances.

\paragraph{Implementation}
The objective of the numerical experiments is to compare the CRA-PI-GNN solver with the PI-GNN solver. 
Thus, we follow the same experimental configuration described as the experiments in  \citet{schuetz2022combinatorial}, employing a simple two-layer \texttt{GCV} and \texttt{GraphSAGE} \citep{hamilton2017inductive} implemented by the Deep Graph Library \citep{wang2019deep}; Refer to Appendix \ref{subsec:app-architecture-gnns} for the detailed architectures.

    We use the AdamW \citep{kingma2014adam} optimizer with a learning rate of $\eta =10^{-4}$ and weight decay of $10^{-2}$.
    The GNNs are trained for up to $5\times 10^{4}$ epochs with early stopping, which monitors the summarized loss function $\sum_{s=1}^{S} \hat{l}(P_{:, s})$ and the entropy term $\Phi(P; \gamma, \alpha)$ with tolerance $10^{-5}$ and patience $10^{3}$ epochs; Further details are provided in Appendix \ref{subsec:app-learning-setting}.
    We set the initial scheduling value to $\gamma(0)=-20$ for the MIS and matching problems, and we set $\gamma(0)=-6$ for the MaxCut problems with the scheduling rate $\varepsilon=10^{-3}$ and curve rate $\alpha=2$ in Eq.~\eqref{eq:penalty-cost-gnn}.
    These values are not necessarily optimal, and refining these parameters can lead to better solutions; Refer to Appendix \ref{subsec:ablation-study} and Appendix \ref{subsec:ablation-curve-rate} for an ablation study of these parameters.
    Once the training process is complete, we apply projection heuristics to map the obtained soft solutions back to discrete solutions using simple projection, where for all $i \in [N]$, we map $p_{\theta, i}$ into $0$ if $p_{\theta, i} \le 0.5$ and $p_{\theta, i}$ into $1$ if $p_{\theta, i} > 0.5$. 
    However, due to the early stopping in Section~\ref{subsec:continuou-relaxation-annealing}, the CRA-PI-GNN solver ensures that for all benchmark CO problems, the soft solution at the end of the training process became 0 or 1 within the 32-bit Floating Point range in Pytorch GPU; thus, it is robust against a given threshold, which we set to $0.5$ in our experiments.
    Additionally, no violations of the constraints were observed in our numerical experiments.
    Thus, following results presented in are feasible solutions. 
\paragraph{Evaluation metrics}
    Following \citet{karalias2020erdos},
    We use the approximation ratio (ApR), formulated as  $\mathrm{ApR}=f(\B{x}; C)/f(\B{x}^{\ast}; C)$, where $\B{x}^{\ast}$ is optimal solution.
    For the MIS problems, we evaluate the ApRs using  the theoretical optimal cost   \citep{barbier2013hard} and the independent set density $\rho$ relative to the theoretical results.
    For the MaxCut problems on RRGs, we adopt the cut ratio $\nu$ against the theoretical upper bound \citep{parisi1980sequence, dembo2017extremal}; see Appendix \ref{sec:theoretical-background-co-problem} for the details.  
    All the results for the MIS and MaxCut problems are summarized based on 5 RRGs with different random seeds. 
    In the case of the MaxCut Gset problem, the ApR is calculated compared to the known best cost functions.
    Regarding the DBM problems, we calculate the ApR against the global optimal, identified using Gurobi 10.0.1 solver with default settings.
    
\subsection{MIS problems}\label{subsec:mis-problems}
\paragraph{Degree dependency of solutions using CRA}
\begin{figure}
\begin{minipage}[c]{0.49\textwidth}
    \centering
    \centerline{\includegraphics[width=\columnwidth]{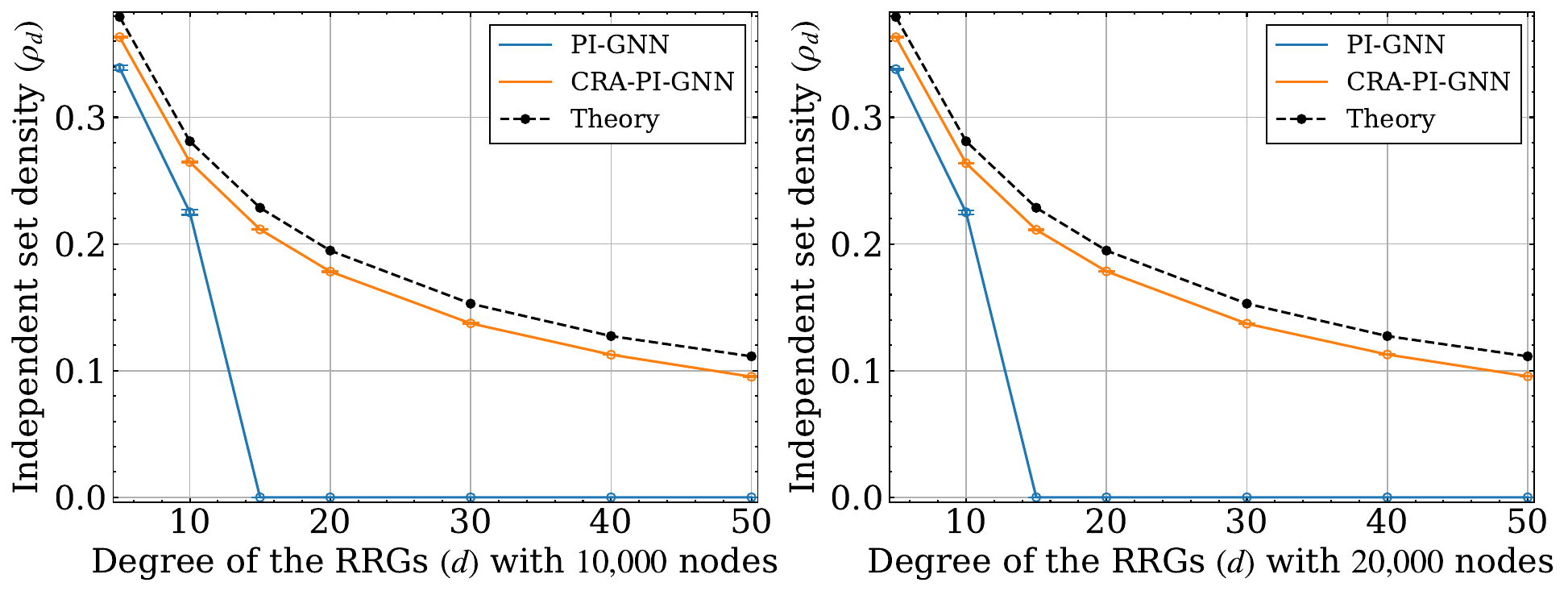}}
    \caption{Independent set density of the MIS problem on $d$-RRG. Results for graphs with $N=10{,}000$ nodes (Left) and $N=20{,}000$ nodes (Right). 
    the dashed lines represent the theoretical results \citep{barbier2013hard}.
    }
\label{fig:mis-degree-result}
\end{minipage}
\hfill
\begin{minipage}[c]{0.49\textwidth}
    \centering
    \centerline{\includegraphics[width=\columnwidth]{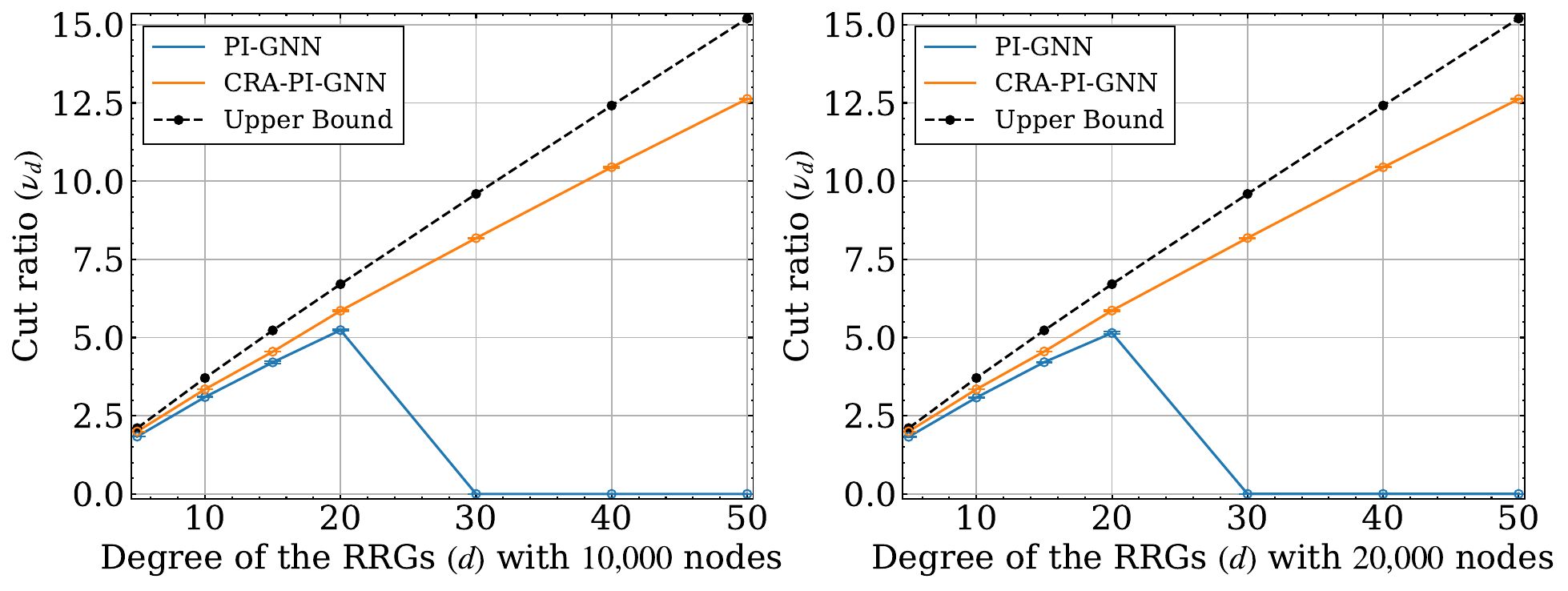}}
    \caption{Cut ratio of the MaxCut problem on $d$-RRG as a function of the degree $d$ Results for $N=10{,}000$ (Left) and $N=20{,}000$ (Right). The dashed lines represents the theoretical upper bounds \citep{parisi1980sequence}.
    }
    \label{fig:maxcut-degree-result}
\end{minipage}
\end{figure}
First, we compare the performance of the PI-GNN and CRA-PI-GNN solvers using GCV, as in \citet{schuetz2022combinatorial}.
Fig.~\ref{fig:mis-degree-result} shows the independent set density $\rho_{d}$ as a function of degree $d$ obtained by the PI-GNN and CRA-PI-GNN solvers compared with the theoretical results \citep{barbier2013hard}. 
Across all degrees $d$, the CRA-PI-GNN solver outperformed the PI-GNN solver and approached the theoretical results, whereas the PI-GNN solver fail to find valid solutions, especially for $d \ge 15$, as pointed by the previous studies \citep{angelini2023modern, wang2023unsupervised}.
\paragraph{Response to \citet{angelini2023modern} and \citet{wang2023unsupervised}}
\begin{table}[tb]
\begin{minipage}[c]{0.55\textwidth}
\centering
\tblcaption{ApR in MIS problems on RRGs with $10{,}000$ nodes and node degree $d=20, 100$. ``CRA'' represents the CRA-PI-GNN solver.}
\begin{tabular}{l|llll}
\toprule
Instance & Method & ApR \\
\hline
\multirow{6}{*}{20-RRG} & RGA & $0.776 \pm 0.001$ \\
 & DGA & $0.891\pm0.001$ \\
 & EGN & $0.775$ (\citeyear{wang2023unsupervised}) \\
 & META-EGN & $0.887$ (\citeyear{wang2023unsupervised}) \\
 & PI-GNN (GCV) & $0.000\pm0.000$ \\
 & PI-GNN (SAGE) & $0.745\pm0.003$ \\
 & \textbf{CRA (GCV)} & $0.937\pm0.002$ \\
 & \textbf{CRA (SAGE)} & $\underline{\mathbf{0.963\pm0.001}}$ \\ \hline
 \multirow{6}{*}{100-RRG} & RGA & $0.663 \pm 0.001$ \\
 & DGA & $0.848 \pm 0.002$ \\
 & PI-GNN (GCV) & $0.000\pm0.000$ \\
 & PI-GNN (SAGE) & $0.000\pm0.000$ \\
 & \textbf{CRA (GCV)} & $0.855 \pm 0.004$ \\
 & \textbf{CRA (SAGE)} & $\underline{\mathbf{0.924\pm 0.001}}$  \\ \hline
\end{tabular}
\label{tab:mis-high-degree}
\end{minipage}
\hfill
\begin{minipage}[c]{0.45\textwidth}
\includegraphics[width=0.99\columnwidth]{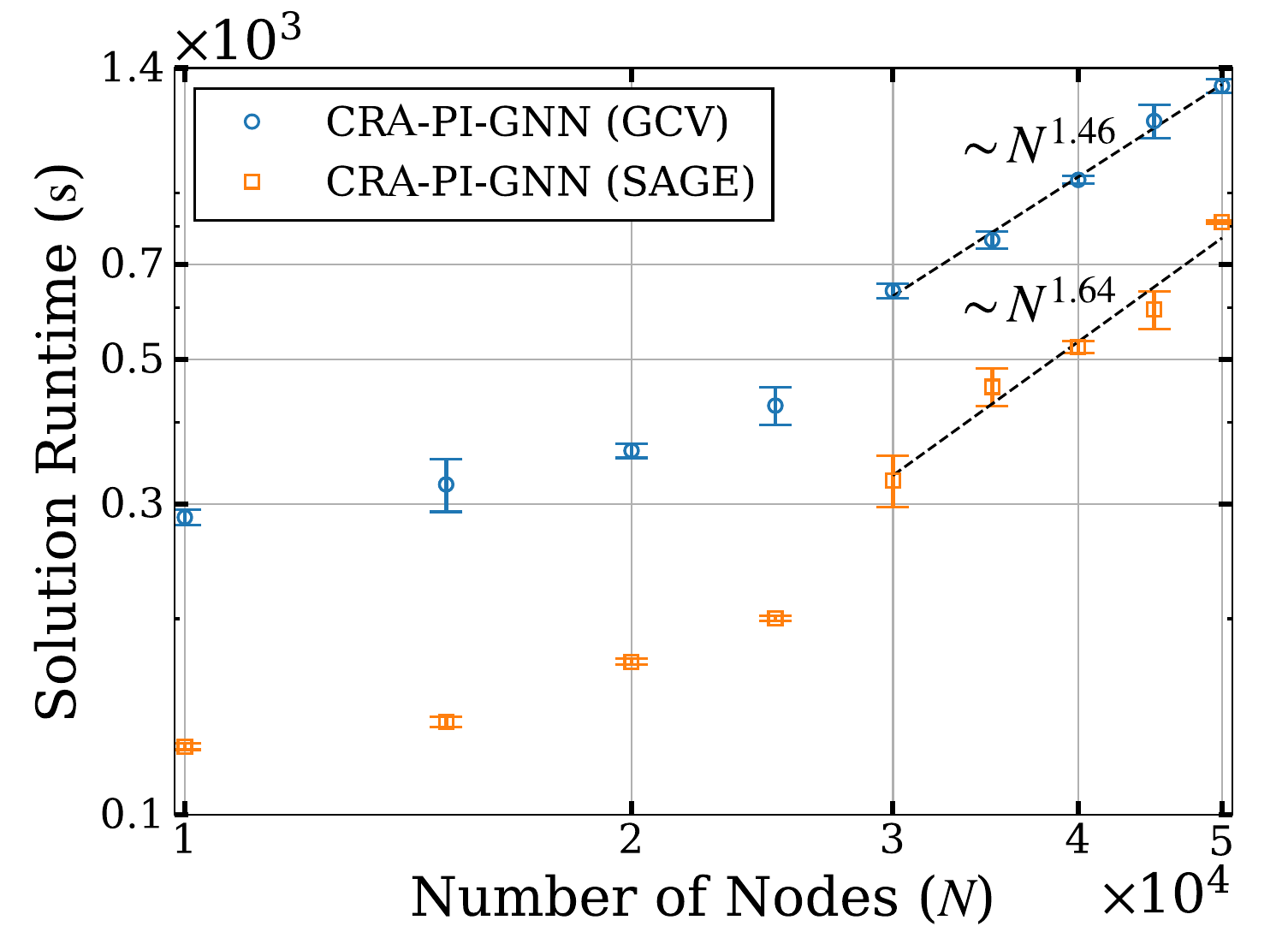}
  \figcaption{(Right) computational runtime (in seconds) of the CRA-PI-GNN solvers with the GraphSage and Conv architectures on $100$-RRG with varying numbers of nodes $N$.Error bars represent the standard deviations of the results.}
\label{fig:computational-scaling}
\end{minipage}
\end{table}

MIS problems on RRGs with a degree $d$ larger than $16$ is known to be hard problems \citep{barbier2013hard}. 
As discussed in Section \ref{sec:how-can-it-fail}, \citet{angelini2023modern, wang2023unsupervised} have posted the optimization concerns on UL-based solvers. 
However, we call these claim into question by substantially outperforming heuristics DGA and RGA for the MIS on graphs with $d=20, 100$, without training/historical instances $\mac{D}=\{G^{\mu}\}_{\mu=1}^{n}$, as shown in Table \ref{tab:mis-high-degree}. 
See Appendix \ref{tab:app-maxcut-gset} for the results of solving all other Gsets, where consistently, CRA-PI-GNN provides better results as well.
A comparison of the sampling-based solvers, RL-based solvers, SL-based solvers, Gurobi, and MIS-specific solvers is presented in Appendix \ref{subsec:app-additional-results-mis}.

\paragraph{Acceleration of learning speed}

We also compared the learning curve between PI-GNN and CRA-PI-GNN solver to confirmed that the CRA-PI-GNN solver does not become trapped in the plateau, $\B{p}_{N}=\B{0}_{N}$, as discussed in Section~\ref{sec:how-can-it-fail}. Fig.~\ref{fig:costdynamics-results} shows the dynamics of the cost functions for the MIS problems with $N=10{,}000$ across $d=3, 5, 20, 100$. 
Across all degrees, CRA-PI-GNN solver achieves a better solution with fewer epochs than PI-GNN solver.
Specifically, PI-GNN solver becomes significantly slower due to getting trapped in the plateau even for graphs with low degrees, such as $d=3, 5$. In contrast, CRA-PI-GNN solver can effectively escape from plateaus through the smoothing and automatic rounding of the penalty term when the negative parameter $\gamma > 0$.
\paragraph{Computational scaling}

\begin{wrapfigure}{r}{0.5\textwidth} 
    \centering
    \vspace{-15pt}
    \includegraphics[width=0.5\textwidth]{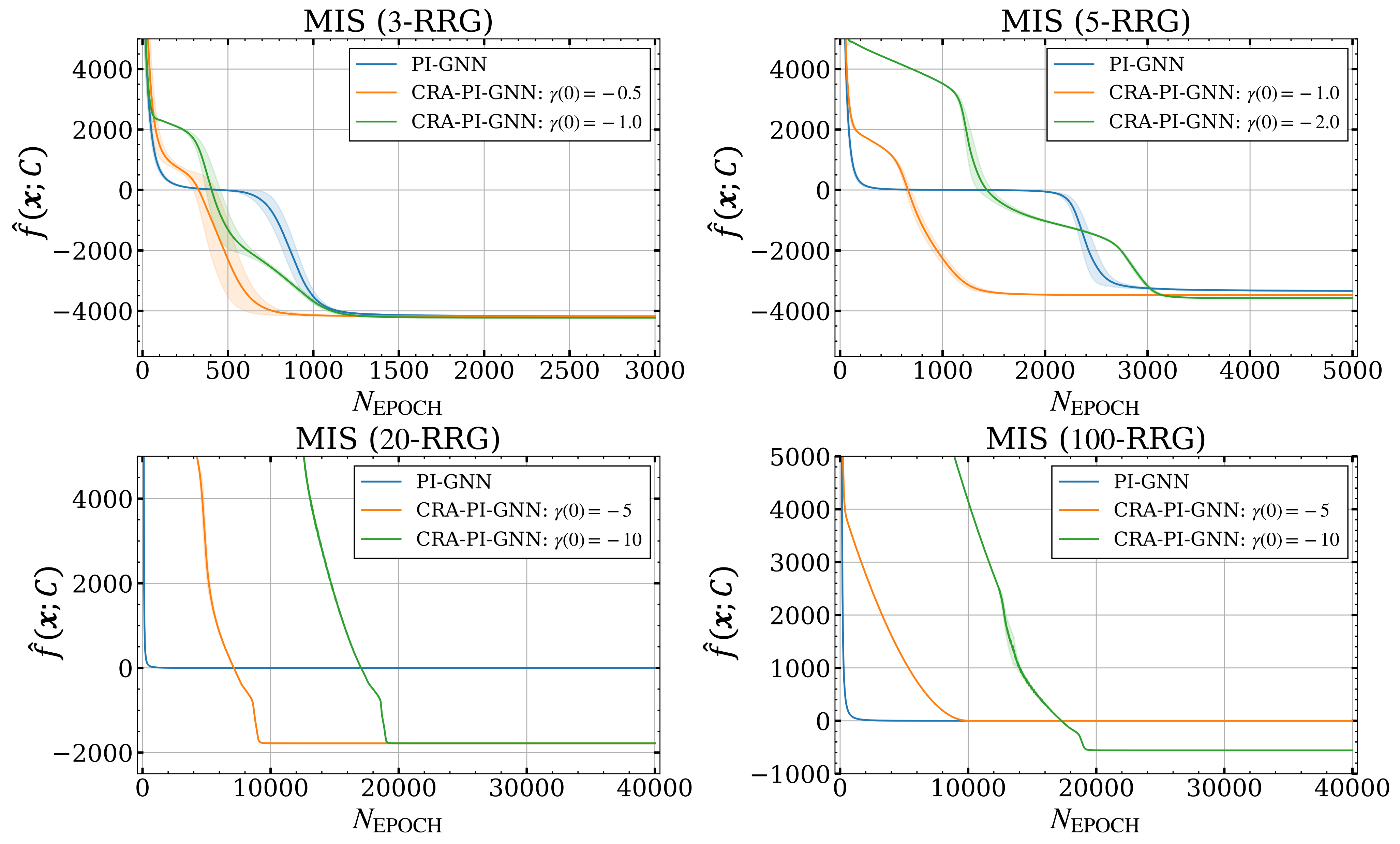}
    \caption{The dynamics of cost function for MIS problems on RRGs with  $N=10{,}000$ nodes varying degrees $d$ as a function of the number of parameters updates $N_{\mathrm{EPOCH}}$. 
    }
    \label{fig:costdynamics-results}
    \vspace{-5pt}
\end{wrapfigure}
We next evaluate the computational scaling of the CRA-PI-GNN solver for MIS problems with large-scale RRGs with a node degree of $100$ in Fig.~\ref{fig:computational-scaling}, following previous studies \citep{schuetz2022combinatorial, wang2023unsupervised}.
Fig.~\ref{fig:computational-scaling} demonstrated a moderate super-linear scaling of the total computational time, approximately $\sim N^{1.4}$ for GCN and $\sim N^{1.7}$ for GraphSage. 
This performance is nearly identical to that of the PI-GNN solver \citep{schuetz2022combinatorial} for problems on RRGs with lower degrees.
It is important note that the runtimes of CRA-PI-GNN solver heavily depend on the  optimizer for GNNs and annealing rate $\varepsilon$; thus this scaling remains largely unchanged for problems other than the MIS on $100$ RRG.
Additionally, CRA demonstrate that the runtime remains nearly constant as graph order and density increase, indicating effective scalability with denser graphs which is presented in Appendix \ref{subsec:app-additional-results-mis}.

\subsection{MaxCut problem}
\paragraph{Degree dependency of solutions}
We first compare the performances of PI-GNN and CRA-PI-GNN solvers with GCV, following \citet{schuetz2022combinatorial}.
Fig.~\ref{fig:maxcut-degree-result} shows the cut ratio $\nu_{d}$ as a function of the degree $d$ compared to the theoretical upper bound \citep{parisi1980sequence, dembo2017extremal}. 
Across all degrees $d$, CRA-PI-GNN solver also outperforms PI-GNN solver, approaching the theoretical upper bound. In contrast, PI-GNN solver fails to find valid solutions for $d>20$ as with the case of the MIS problems in Section \ref{subsec:mis-problems}.

\paragraph{Standard MaxCut benchmark test}
\begin{wrapfigure}{r}{0.41\textwidth} 
      \centering
      \includegraphics[width=0.41\textwidth]{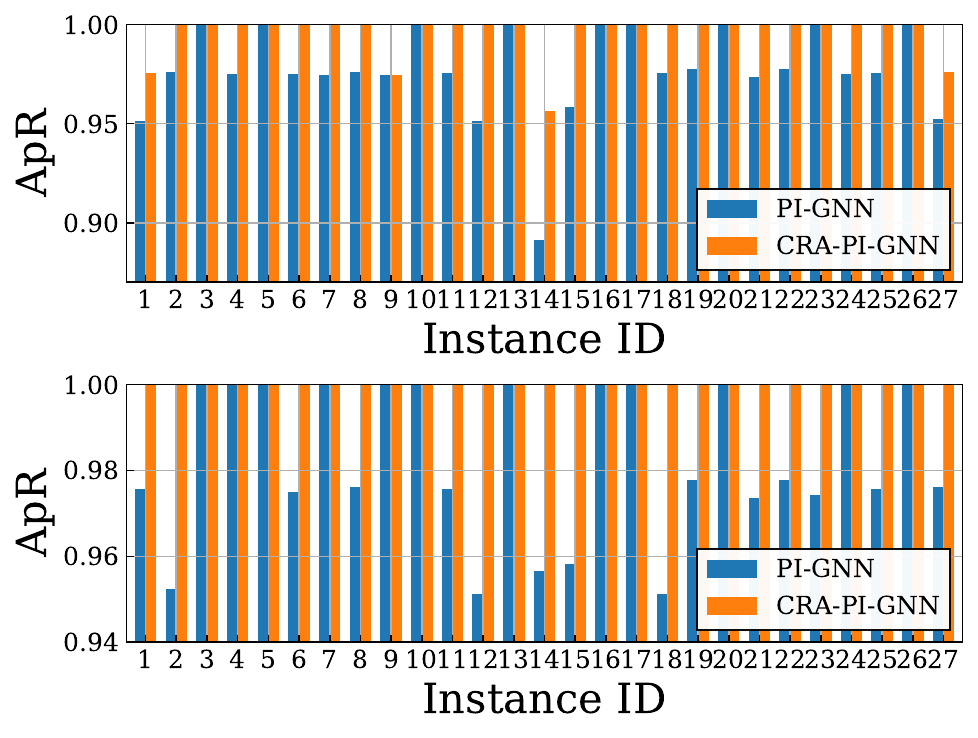}
      \caption{ApR on DBM  problems.}
      \label{fig:matching_result}
    \end{wrapfigure}
Following  \citet{schuetz2022combinatorial}, 
we next conducted additional experiments on standard MaxCut benchmark instances based on the publicly available Gset dataset \citep{gsetdataset}, which is commonly used to evaluate MaxCut algorithms. 
Here, we provide benchmark results for seven distinct graphs with thousands of nodes, including Erd\"{o}s-Renyi graphs with uniform edge probability, graphs in which the connectivity decays gradually from node $1$ to $N$, $4$-regular toroidal graphs, and a very large Gset instance with $N=10{,}000$ nodes. 
Table~\ref{tab:maxcut-gset} shows, across all problems, CRA-PI-GNN solver outperforms both the PI-GNN, RUN-CSP solvers and other greedy algorithm.
See Appendix \ref{tab:app-maxcut-gset} for the results of solving all other Gsets, where consistently, CRA-PI-GNN provides better results as well.
\begin{table}[tb]
    \caption{ApR for MaxCut on Gset}
    \label{tab:maxcut-gset}
        \centering
        \begin{tabular}{lccccccc}
                \toprule
                GRAPH & (NODES, EDGES) & GREEDY & SDP & RUN-CSP & PI-GNN & \textbf{CRA}  \\
                \midrule
                G14 & ($800$, $4{,}694$) & $0.946$ & $0.970$ &$0.960$ & $0.988$ & $\underline{\mathbf{0.994}}$  \\
                G15 & ($800$, $4{,}661$) & $0.939$ & $0.958$ &$0.960$ & $0.980$  & $\underline{\mathbf{0.992}}$  \\
                G22 & ($2{,}000$, $19{,}990$) & $0.923$ & $0.77$ & $0.975$ & $0.987$ & $\underline{\mathbf{0.998}}$   \\
                G49 & ($3{,}000$, $6{,}000$) & $\underline{\mathbf{1.000}}$& $\underline{\mathbf{1,00}}$ & $\underline{\mathbf{1.000}}$ &  $0.986$ & $\mathbf{\underline{1.000}}$  \\
                G50 & ($3{,}000$, $6{,}000$) & $\mathbf{\underline{1.000}}$ & $\mathbf{\underline{1.000}}$ & $\mathbf{\underline{1.000}}$ & $0.990$  & $\mathbf{\underline{1.000}}$  \\
                G55 & ($5{,}000$, $12{,}468$) &$0.892$ & $-$ & $0.982$ & $0.983$ & $\underline{\mathbf{0.991}}$  \\
                G70 & ($10{,}000$, $9{,}999$) & $0.886$ & $-$ & $0.970$ & $0.982$ & $\underline{\mathbf{0.992}}$  \\
                \bottomrule
            \end{tabular}
\end{table}

\subsection{Diverse bipartite matching}\label{subsec:diverse-bipartite-matching}
    To evaluate the applicability of the CRA-PI-GNN solver to more practical problems not on graphs, we conducted experiments on DBM problems \citep{ferber2020mipaal, mulamba2020contrastive, mandi2022decision}; refer to Appendix \ref{sec:theoretical-background-co-problem} for details. 
    This problems consists of $27$ distinct instances with varying properties, and each instance comprises $100$ nodes representing scientific publications, divided into two groups of $50$ nodes $N_{1}$ and $N_{2}$. 
    The optimization is formulated as follows:
    \begin{multline*}
    \label{eq:matching-cost-function}
        l(\B{x}; C, M, \B{\lambda}) = - \textstyle\sum_{ij} C_{ij} x_{ij} + \lambda_{1} \textstyle\sum_{i}\mathrm{ReLU}\Big(\textstyle\sum_{j} x_{ij}-1 \Big) + \lambda_{2} \textstyle\sum_{j}\mathrm{ReLU}\Big(\sum_{i} x_{ij} - 1\Big) \\
        + \lambda_{3} \mathrm{ReLU}\Big(p\textstyle\sum_{ij} x_{ij} - \textstyle\sum_{ij} M_{ij} x_{ij} \Big) + \lambda_{4} \mathrm{ReLU}\Big(q\sum_{ij} x_{ij} - \sum_{ij} (1-M_{ij}) x_{ij} \Big),
    \end{multline*}
    where $C \in \mab{R}^{N_{1} \times N_{2}}$ represents the likelihood of a link between each pair of nodes, an indicator $M_{ij}$ is set to $0$ if article $i$ and $j$ share the same subject field ($1$ otherwise) $\forall i \in N_{1}$, and $j \in N_{2}$. The parameters $p, q \in [0, 1]$ represent the probability of pairs sharing their field and of unrelated pairs, respectively. 
    As in \citet{mandi2022decision}, we explore two variations of this problem, with $p=q=$ being $25$\% and $5$\%, respectively, and these variations are referred to as Matching-1 and Matching-2, respectively. In this experiment, we set $\lambda_{1}=\lambda_{2}=10$ and $\lambda_{3}=\lambda_{4}=25$.
    Fig~\ref{fig:matching_result} shows that the CRA-PI-GNN solver can find better solutions across all instances.

\section{Conclusion}\label{sec:conclusion}
    This study proposes CRA strategy to address the both optimization and rounding issue in UL-based solvers. 
    CRA strategy introduces a penalty term that dynamically shifts from prioritizing continuous solutions, where the non-convexity of the objective function is effectively smoothed, to enforcing discreteness, thereby eliminating artificial rounding.
    Experimental results demonstrate that CRA-PI-GNN solver significantly outperforms both PI-GNN solver and greedy algorithms across various complex CO problems, including MIS, MaxCut, and DBM problems. 
    CRA approach not only enhances solution quality but also accelerates the learning process.

    \paragraph{Limitation}
    In this numerical experiments, most hyperparameters were fixed to their default values, as outlined in  Section \ref{sub-sec:config}, with minimal tuning. However, tuning may be necessary for certain problems or to further enhance performance.

\bibliographystyle{unsrtnat}
\bibliography{ref}

\begin{thebibliography}{60}
\providecommand{\natexlab}[1]{#1}
\providecommand{\url}[1]{\texttt{#1}}
\expandafter\ifx\csname urlstyle\endcsname\relax
  \providecommand{\doi}[1]{doi: #1}\else
  \providecommand{\doi}{doi: \begingroup \urlstyle{rm}\Url}\fi

\bibitem[Papadimitriou and Steiglitz(1998)]{papadimitriou1998combinatorial}
Christos~H Papadimitriou and Kenneth Steiglitz.
\newblock \emph{Combinatorial optimization: algorithms and complexity}.
\newblock Courier Corporation, 1998.

\bibitem[Hudson et~al.(2021)Hudson, Li, Malencia, and Prorok]{hudson2021graph}
Benjamin Hudson, Qingbiao Li, Matthew Malencia, and Amanda Prorok.
\newblock Graph neural network guided local search for the traveling salesperson problem.
\newblock \emph{arXiv preprint arXiv:2110.05291}, 2021.

\bibitem[Joshi et~al.(2019)Joshi, Laurent, and Bresson]{joshi2019efficient}
Chaitanya~K Joshi, Thomas Laurent, and Xavier Bresson.
\newblock An efficient graph convolutional network technique for the travelling salesman problem.
\newblock \emph{arXiv preprint arXiv:1906.01227}, 2019.

\bibitem[Gasse et~al.(2019)Gasse, Ch{\'e}telat, Ferroni, Charlin, and Lodi]{gasse2019exact}
Maxime Gasse, Didier Ch{\'e}telat, Nicola Ferroni, Laurent Charlin, and Andrea Lodi.
\newblock Exact combinatorial optimization with graph convolutional neural networks.
\newblock \emph{Advances in neural information processing systems}, 32, 2019.

\bibitem[Selsam et~al.(2018)Selsam, Lamm, B{\"u}nz, Liang, de~Moura, and Dill]{selsam2018learning}
Daniel Selsam, Matthew Lamm, Benedikt B{\"u}nz, Percy Liang, Leonardo de~Moura, and David~L Dill.
\newblock Learning a sat solver from single-bit supervision.
\newblock \emph{arXiv preprint arXiv:1802.03685}, 2018.

\bibitem[Khalil et~al.(2016)Khalil, Le~Bodic, Song, Nemhauser, and Dilkina]{khalil2016learning}
Elias Khalil, Pierre Le~Bodic, Le~Song, George Nemhauser, and Bistra Dilkina.
\newblock Learning to branch in mixed integer programming.
\newblock In \emph{Proceedings of the AAAI Conference on Artificial Intelligence}, volume~30, 2016.

\bibitem[Yehuda et~al.(2020)Yehuda, Gabel, and Schuster]{yehuda2020s}
Gal Yehuda, Moshe Gabel, and Assaf Schuster.
\newblock It’s not what machines can learn, it’s what we cannot teach.
\newblock In \emph{International conference on machine learning}, pages 10831--10841. PMLR, 2020.

\bibitem[Yao et~al.(2019)Yao, Bandeira, and Villar]{yao2019experimental}
Weichi Yao, Afonso~S Bandeira, and Soledad Villar.
\newblock Experimental performance of graph neural networks on random instances of max-cut.
\newblock In \emph{Wavelets and Sparsity XVIII}, volume 11138, pages 242--251. SPIE, 2019.

\bibitem[Chen and Tian(2019)]{chen2019learning}
Xinyun Chen and Yuandong Tian.
\newblock Learning to perform local rewriting for combinatorial optimization.
\newblock \emph{Advances in Neural Information Processing Systems}, 32, 2019.

\bibitem[Yolcu and P{\'o}czos(2019)]{yolcu2019learning}
Emre Yolcu and Barnab{\'a}s P{\'o}czos.
\newblock Learning local search heuristics for boolean satisfiability.
\newblock \emph{Advances in Neural Information Processing Systems}, 32, 2019.

\bibitem[Nazari et~al.(2018)Nazari, Oroojlooy, Snyder, and Tak{\'a}c]{nazari2018reinforcement}
Mohammadreza Nazari, Afshin Oroojlooy, Lawrence Snyder, and Martin Tak{\'a}c.
\newblock Reinforcement learning for solving the vehicle routing problem.
\newblock \emph{Advances in neural information processing systems}, 31, 2018.

\bibitem[Khalil et~al.(2017)Khalil, Dai, Zhang, Dilkina, and Song]{khalil2017learning}
Elias Khalil, Hanjun Dai, Yuyu Zhang, Bistra Dilkina, and Le~Song.
\newblock Learning combinatorial optimization algorithms over graphs.
\newblock \emph{Advances in neural information processing systems}, 30, 2017.

\bibitem[Bello et~al.(2016)Bello, Pham, Le, Norouzi, and Bengio]{bello2016neural}
Irwan Bello, Hieu Pham, Quoc~V Le, Mohammad Norouzi, and Samy Bengio.
\newblock Neural combinatorial optimization with reinforcement learning.
\newblock \emph{arXiv preprint arXiv:1611.09940}, 2016.

\bibitem[Mnih et~al.(2015)Mnih, Kavukcuoglu, Silver, Rusu, Veness, Bellemare, Graves, Riedmiller, Fidjeland, Ostrovski, et~al.]{mnih2015human}
Volodymyr Mnih, Koray Kavukcuoglu, David Silver, Andrei~A Rusu, Joel Veness, Marc~G Bellemare, Alex Graves, Martin Riedmiller, Andreas~K Fidjeland, Georg Ostrovski, et~al.
\newblock Human-level control through deep reinforcement learning.
\newblock \emph{nature}, 518\penalty0 (7540):\penalty0 529--533, 2015.

\bibitem[Tang et~al.(2017)Tang, Houthooft, Foote, Stooke, Xi~Chen, Duan, Schulman, DeTurck, and Abbeel]{tang2017exploration}
Haoran Tang, Rein Houthooft, Davis Foote, Adam Stooke, OpenAI Xi~Chen, Yan Duan, John Schulman, Filip DeTurck, and Pieter Abbeel.
\newblock \# exploration: A study of count-based exploration for deep reinforcement learning.
\newblock \emph{Advances in neural information processing systems}, 30, 2017.

\bibitem[Espeholt et~al.(2018)Espeholt, Soyer, Munos, Simonyan, Mnih, Ward, Doron, Firoiu, Harley, Dunning, et~al.]{espeholt2018impala}
Lasse Espeholt, Hubert Soyer, Remi Munos, Karen Simonyan, Vlad Mnih, Tom Ward, Yotam Doron, Vlad Firoiu, Tim Harley, Iain Dunning, et~al.
\newblock Impala: Scalable distributed deep-rl with importance weighted actor-learner architectures.
\newblock In \emph{International conference on machine learning}, pages 1407--1416. PMLR, 2018.

\bibitem[Schuetz et~al.(2022{\natexlab{a}})Schuetz, Brubaker, and Katzgraber]{schuetz2022combinatorial}
Martin~JA Schuetz, J~Kyle Brubaker, and Helmut~G Katzgraber.
\newblock Combinatorial optimization with physics-inspired graph neural networks.
\newblock \emph{Nature Machine Intelligence}, 4\penalty0 (4):\penalty0 367--377, 2022{\natexlab{a}}.

\bibitem[Karalias and Loukas(2020)]{karalias2020erdos}
Nikolaos Karalias and Andreas Loukas.
\newblock Erdos goes neural: an unsupervised learning framework for combinatorial optimization on graphs.
\newblock \emph{Advances in Neural Information Processing Systems}, 33:\penalty0 6659--6672, 2020.

\bibitem[Amizadeh et~al.(2018)Amizadeh, Matusevych, and Weimer]{amizadeh2018learning}
Saeed Amizadeh, Sergiy Matusevych, and Markus Weimer.
\newblock Learning to solve circuit-sat: An unsupervised differentiable approach.
\newblock In \emph{International Conference on Learning Representations}, 2018.

\bibitem[Angelini and Ricci-Tersenghi(2023)]{angelini2023modern}
Maria~Chiara Angelini and Federico Ricci-Tersenghi.
\newblock Modern graph neural networks do worse than classical greedy algorithms in solving combinatorial optimization problems like maximum independent set.
\newblock \emph{Nature Machine Intelligence}, 5\penalty0 (1):\penalty0 29--31, 2023.

\bibitem[Angelini and Ricci-Tersenghi(2019)]{angelini2019monte}
Maria~Chiara Angelini and Federico Ricci-Tersenghi.
\newblock Monte carlo algorithms are very effective in finding the largest independent set in sparse random graphs.
\newblock \emph{Physical Review E}, 100\penalty0 (1):\penalty0 013302, 2019.

\bibitem[Wang and Li(2023)]{wang2023unsupervised}
Haoyu Wang and Pan Li.
\newblock Unsupervised learning for combinatorial optimization needs meta-learning.
\newblock \emph{arXiv preprint arXiv:2301.03116}, 2023.

\bibitem[Hoffman and Kruskal(2010)]{hoffman2010integral}
Alan~J Hoffman and Joseph~B Kruskal.
\newblock Integral boundary points of convex polyhedra.
\newblock \emph{50 Years of Integer Programming 1958-2008: From the Early Years to the State-of-the-Art}, pages 49--76, 2010.

\bibitem[Nemhauser and Trotter~Jr(1974)]{nemhauser1974properties}
George~L Nemhauser and Leslie~E Trotter~Jr.
\newblock Properties of vertex packing and independence system polyhedra.
\newblock \emph{Mathematical programming}, 6\penalty0 (1):\penalty0 48--61, 1974.

\bibitem[Schuetz et~al.(2022{\natexlab{b}})Schuetz, Brubaker, Zhu, and Katzgraber]{schuetz2022graph}
Martin~JA Schuetz, J~Kyle Brubaker, Zhihuai Zhu, and Helmut~G Katzgraber.
\newblock Graph coloring with physics-inspired graph neural networks.
\newblock \emph{Physical Review Research}, 4\penalty0 (4):\penalty0 043131, 2022{\natexlab{b}}.

\bibitem[Wang et~al.(2022)Wang, Wu, Yang, Hao, and Li]{wang2022unsupervised}
Haoyu~Peter Wang, Nan Wu, Hang Yang, Cong Hao, and Pan Li.
\newblock Unsupervised learning for combinatorial optimization with principled objective relaxation.
\newblock \emph{Advances in Neural Information Processing Systems}, 35:\penalty0 31444--31458, 2022.

\bibitem[Lov{\'a}sz(1979)]{lovasz1979shannon}
L{\'a}szl{\'o} Lov{\'a}sz.
\newblock On the shannon capacity of a graph.
\newblock \emph{IEEE Transactions on Information theory}, 25\penalty0 (1):\penalty0 1--7, 1979.

\bibitem[Goemans and Williamson(1995)]{goemans1995improved}
Michel~X Goemans and David~P Williamson.
\newblock Improved approximation algorithms for maximum cut and satisfiability problems using semidefinite programming.
\newblock \emph{Journal of the ACM (JACM)}, 42\penalty0 (6):\penalty0 1115--1145, 1995.

\bibitem[Von~Luxburg(2007)]{von2007tutorial}
Ulrike Von~Luxburg.
\newblock A tutorial on spectral clustering.
\newblock \emph{Statistics and computing}, 17:\penalty0 395--416, 2007.

\bibitem[Smith et~al.(1997)Smith, Coit, Baeck, Fogel, and Michalewicz]{smith1997penalty}
Alice~E Smith, David~W Coit, Thomas Baeck, David Fogel, and Zbigniew Michalewicz.
\newblock Penalty functions.
\newblock \emph{Handbook of evolutionary computation}, 97\penalty0 (1):\penalty0 C5, 1997.

\bibitem[Karp(2010)]{karp2010reducibility}
Richard~M Karp.
\newblock \emph{Reducibility among combinatorial problems}.
\newblock Springer, 2010.

\bibitem[Sun et~al.(2022)Sun, Guha, and Dai]{sun2022annealed}
Haoran Sun, Etash~K Guha, and Hanjun Dai.
\newblock Annealed training for combinatorial optimization on graphs.
\newblock \emph{arXiv preprint arXiv:2207.11542}, 2022.

\bibitem[Sanokowski et~al.(2024)Sanokowski, Berghammer, Hochreiter, and Lehner]{sanokowski2024variational}
Sebastian Sanokowski, Wilhelm Berghammer, Sepp Hochreiter, and Sebastian Lehner.
\newblock Variational annealing on graphs for combinatorial optimization.
\newblock \emph{Advances in Neural Information Processing Systems}, 36, 2024.

\bibitem[Toenshoff et~al.(2019)Toenshoff, Ritzert, Wolf, and Grohe]{toenshoff2019run}
Jan Toenshoff, Martin Ritzert, Hinrikus Wolf, and Martin Grohe.
\newblock Run-csp: unsupervised learning of message passing networks for binary constraint satisfaction problems.
\newblock \emph{CoRR, abs/1909.08387}, 2019.

\bibitem[Boettcher(2023)]{boettcher2023inability}
Stefan Boettcher.
\newblock Inability of a graph neural network heuristic to outperform greedy algorithms in solving combinatorial optimization problems.
\newblock \emph{Nature Machine Intelligence}, 5\penalty0 (1):\penalty0 24--25, 2023.

\bibitem[Lin et~al.(2023)Lin, Yang, Zhang, and Zhang]{lin2023continuation}
Xi~Lin, Zhiyuan Yang, Xiaoyuan Zhang, and Qingfu Zhang.
\newblock Continuation path learning for homotopy optimization.
\newblock In \emph{International Conference on Machine Learning}, pages 21288--21311. PMLR, 2023.

\bibitem[Kirkpatrick et~al.(1983)Kirkpatrick, Gelatt~Jr, and Vecchi]{kirkpatrick1983optimization}
Scott Kirkpatrick, C~Daniel Gelatt~Jr, and Mario~P Vecchi.
\newblock Optimization by simulated annealing.
\newblock \emph{science}, 220\penalty0 (4598):\penalty0 671--680, 1983.

\bibitem[Mehta(2019)]{cvx_graph_algorithms}
Hermish Mehta.
\newblock Cvx graph algorithms.
\newblock \url{https://github.com/hermish/cvx-graph-algorithms}, 2019.

\bibitem[Hamilton et~al.(2017)Hamilton, Ying, and Leskovec]{hamilton2017inductive}
Will Hamilton, Zhitao Ying, and Jure Leskovec.
\newblock Inductive representation learning on large graphs.
\newblock \emph{Advances in neural information processing systems}, 30, 2017.

\bibitem[Wang et~al.(2019)Wang, Zheng, Ye, Gan, Li, Song, Zhou, Ma, Yu, Gai, et~al.]{wang2019deep}
Minjie Wang, Da~Zheng, Zihao Ye, Quan Gan, Mufei Li, Xiang Song, Jinjing Zhou, Chao Ma, Lingfan Yu, Yu~Gai, et~al.
\newblock Deep graph library: A graph-centric, highly-performant package for graph neural networks.
\newblock \emph{arXiv preprint arXiv:1909.01315}, 2019.

\bibitem[Kingma and Ba(2014)]{kingma2014adam}
Diederik~P Kingma and Jimmy Ba.
\newblock Adam: A method for stochastic optimization.
\newblock \emph{arXiv preprint arXiv:1412.6980}, 2014.

\bibitem[Barbier et~al.(2013)Barbier, Krzakala, Zdeborov{\'a}, and Zhang]{barbier2013hard}
Jean Barbier, Florent Krzakala, Lenka Zdeborov{\'a}, and Pan Zhang.
\newblock The hard-core model on random graphs revisited.
\newblock In \emph{Journal of Physics: Conference Series}, volume 473, page 012021. IOP Publishing, 2013.

\bibitem[Parisi(1980)]{parisi1980sequence}
Giorgio Parisi.
\newblock A sequence of approximated solutions to the sk model for spin glasses.
\newblock \emph{Journal of Physics A: Mathematical and General}, 13\penalty0 (4):\penalty0 L115, 1980.

\bibitem[Dembo et~al.(2017)Dembo, Montanari, and Sen]{dembo2017extremal}
Amir Dembo, Andrea Montanari, and Subhabrata Sen.
\newblock Extremal cuts of sparse random graphs.
\newblock 2017.

\bibitem[Ye(2003)]{gsetdataset}
Y.~Ye.
\newblock The gset dataset.
\newblock \url{https://web.stanford.edu/~yyye/yyye/Gset/}, 2003.

\bibitem[Ferber et~al.(2020)Ferber, Wilder, Dilkina, and Tambe]{ferber2020mipaal}
Aaron Ferber, Bryan Wilder, Bistra Dilkina, and Milind Tambe.
\newblock Mipaal: Mixed integer program as a layer.
\newblock In \emph{Proceedings of the AAAI Conference on Artificial Intelligence}, volume~34, pages 1504--1511, 2020.

\bibitem[Mulamba et~al.(2020)Mulamba, Mandi, Diligenti, Lombardi, Bucarey, and Guns]{mulamba2020contrastive}
Maxime Mulamba, Jayanta Mandi, Michelangelo Diligenti, Michele Lombardi, Victor Bucarey, and Tias Guns.
\newblock Contrastive losses and solution caching for predict-and-optimize.
\newblock \emph{arXiv preprint arXiv:2011.05354}, 2020.

\bibitem[Mandi et~al.(2022)Mandi, Bucarey, Tchomba, and Guns]{mandi2022decision}
Jayanta Mandi, V{\i}ctor Bucarey, Maxime Mulamba~Ke Tchomba, and Tias Guns.
\newblock Decision-focused learning: through the lens of learning to rank.
\newblock In \emph{International Conference on Machine Learning}, pages 14935--14947. PMLR, 2022.

\bibitem[Bayati et~al.(2010)Bayati, Gamarnik, and Tetali]{bayati2010combinatorial}
Mohsen Bayati, David Gamarnik, and Prasad Tetali.
\newblock Combinatorial approach to the interpolation method and scaling limits in sparse random graphs.
\newblock In \emph{Proceedings of the forty-second ACM symposium on Theory of computing}, pages 105--114, 2010.

\bibitem[Coja-Oghlan and Efthymiou(2015)]{coja2015independent}
Amin Coja-Oghlan and Charilaos Efthymiou.
\newblock On independent sets in random graphs.
\newblock \emph{Random Structures \& Algorithms}, 47\penalty0 (3):\penalty0 436--486, 2015.

\bibitem[Alidaee et~al.(1994)Alidaee, Kochenberger, and Ahmadian]{alidaee19940}
Bahram Alidaee, Gary~A Kochenberger, and Ahmad Ahmadian.
\newblock 0-1 quadratic programming approach for optimum solutions of two scheduling problems.
\newblock \emph{International Journal of Systems Science}, 25\penalty0 (2):\penalty0 401--408, 1994.

\bibitem[Neven et~al.(2008)Neven, Rose, and Macready]{neven2008image}
Hartmut Neven, Geordie Rose, and William~G Macready.
\newblock Image recognition with an adiabatic quantum computer i. mapping to quadratic unconstrained binary optimization.
\newblock \emph{arXiv preprint arXiv:0804.4457}, 2008.

\bibitem[Deza and Laurent(1994)]{deza1994applications}
Michel Deza and Monique Laurent.
\newblock Applications of cut polyhedra—ii.
\newblock \emph{Journal of Computational and Applied Mathematics}, 55\penalty0 (2):\penalty0 217--247, 1994.

\bibitem[Sen et~al.(2008)Sen, Namata, Bilgic, Getoor, Galligher, and Eliassi-Rad]{sen2008collective}
Prithviraj Sen, Galileo Namata, Mustafa Bilgic, Lise Getoor, Brian Galligher, and Tina Eliassi-Rad.
\newblock Collective classification in network data.
\newblock \emph{AI magazine}, 29\penalty0 (3):\penalty0 93--93, 2008.

\bibitem[Gilmer et~al.(2017)Gilmer, Schoenholz, Riley, Vinyals, and Dahl]{gilmer2017neural}
Justin Gilmer, Samuel~S Schoenholz, Patrick~F Riley, Oriol Vinyals, and George~E Dahl.
\newblock Neural message passing for quantum chemistry.
\newblock In \emph{International conference on machine learning}, pages 1263--1272. PMLR, 2017.

\bibitem[Scarselli et~al.(2008)Scarselli, Gori, Tsoi, Hagenbuchner, and Monfardini]{scarselli2008graph}
Franco Scarselli, Marco Gori, Ah~Chung Tsoi, Markus Hagenbuchner, and Gabriele Monfardini.
\newblock The graph neural network model.
\newblock \emph{IEEE transactions on neural networks}, 20\penalty0 (1):\penalty0 61--80, 2008.

\bibitem[Qiu et~al.(2022)Qiu, Sun, and Yang]{qiu2022dimes}
Ruizhong Qiu, Zhiqing Sun, and Yiming Yang.
\newblock {DIMES}: A differentiable meta solver for combinatorial optimization problems.
\newblock In \emph{Advances in Neural Information Processing Systems 35}, 2022.

\bibitem[Sun et~al.(2023)Sun, Goshvadi, Nova, Schuurmans, and Dai]{sun2023revisiting}
Haoran Sun, Katayoon Goshvadi, Azade Nova, Dale Schuurmans, and Hanjun Dai.
\newblock Revisiting sampling for combinatorial optimization.
\newblock In \emph{International Conference on Machine Learning}, pages 32859--32874. PMLR, 2023.

\bibitem[Goshvadi et~al.(2023)Goshvadi, Sun, Liu, Nova, Zhang, Grathwohl, Schuurmans, and Dai]{NEURIPS2023_f9ad87c1}
Katayoon Goshvadi, Haoran Sun, Xingchao Liu, Azade Nova, Ruqi Zhang, Will Grathwohl, Dale Schuurmans, and Hanjun Dai.
\newblock Discs: A benchmark for discrete sampling.
\newblock In A.~Oh, T.~Naumann, A.~Globerson, K.~Saenko, M.~Hardt, and S.~Levine, editors, \emph{Advances in Neural Information Processing Systems}, volume~36, pages 79035--79066. Curran Associates, Inc., 2023.
\newblock URL \url{https://proceedings.neurips.cc/paper_files/paper/2023/file/f9ad87c1ebbae8a3555adb31dbcacf44-Paper-Datasets_and_Benchmarks.pdf}.

\bibitem[Hoos and St{\"u}tzle(2000)]{hoos2000satlib}
Holger~H Hoos and Thomas St{\"u}tzle.
\newblock Satlib: An online resource for research on sat.
\newblock \emph{Sat}, 2000:\penalty0 283--292, 2000.

\end{thebibliography}

\newpage
\appendix

\section{Overview}
This supplementary material provides extended explanations, implementation details, and additional results.

\section{Derivation}\label{sec:derivaiton}

\subsection{Proof of Theorem \ref{theorem:gamma-annealing-limit}}\label{sec:proof-theorem}

First, we present three lemmas, and then we demonstrate Theorem \ref{theorem:gamma-annealing-limit} based on these lemmas.
\begin{lemma}
\label{lemma:app-component-lemma}
For any even natural number $\alpha \in \{2n \mid n \in \mab{N}_{+}\}$, the function $\phi(p) = 1 - (2p-1)^{\alpha}$ defined on $[0, 1]$ achieves its maximum value of $1$ when $p=1/2$ and its minimum value of $0$ when $p=0$ or $p=1$.
\end{lemma}
\begin{proof}
The derivative of $\phi(p)$ relative to $p$ is $\nicefrac{d\phi(p)}{dp} = -2\alpha(2p-1)$, which is zero when $p=1/2$.
This is a point where the function is maximized because the second derivative $\nicefrac{d^{2}\phi(p)}{dp^{2}} = -4\alpha \le 0$.
In addition, this function is concave and symmetric relative to $p=1/2$ because $\alpha$ is an even natural number, i.e., $\phi(p)=\phi(1-p)$, thereby achieving its minimum value of $0$ when $p=0$ or $p=1$.
\end{proof}

\begin{lemma}
\label{lemma:sum-penalty-enforcing}
For any even natural number $\alpha \in \{2n \mid n \in \mab{N}_{+}\}$, if $\gamma \to +\infty$, minimizing the penalty term $\gamma \Phi(\B{p})= \gamma \sum_{i=1}^{N} (1-(2p_{i}-1)^{\alpha}) = \gamma \sum_{i=1}^{N} \phi(p_{i})$ enforces that the for all $i \in [N]$, $p_{i}$ is either $0$ or $1$ and, if $\gamma \to - \infty$, the penalty term enforces $\B{p}=\B{1}_{N}/2$.
\end{lemma}

\begin{proof}
From Lemma \ref{lemma:app-component-lemma}, as $\gamma \to +\infty$, $\phi(p)$ is minimal value when, for any $i \in [N]$, $p_{i}=0$ or $p_{i}=1$. 
As $\gamma \to -\infty$, $\phi(p; \alpha, \gamma)$ is minimal value when, for any $i \in [N]$, $p_{i}=1/2$.
\end{proof}

\begin{lemma}
\label{lemma:sum-penalty-convex}
For any even number $\alpha\in \{2n \mid n \in \mab{N}_{+}\}$, $\gamma \Phi(\B{p})$ is concave when $\lambda > 0$ and is a convex function when $\lambda <0$.
\end{lemma}

\begin{proof}
    Note that $\gamma \Phi(\B{p}) = \gamma \sum_{i=1}^{N} \phi(p_{i}) = \gamma \sum_{i=1}^{N} (1-(2p_{i}-1)^{\alpha})$ is separable across its components $p_{i}$. 
    Thus, it is sufficient to prove that each $\gamma \phi_{i}(p_{i}; \alpha)$ is concave or convex in $p_{i}$ because the sum of the concave or convex functions is also concave (and vice versa). Thus, we consider the second derivative of $\gamma \phi_{i}(p_{i})$ with respect to $p_{i}$:
    \begin{equation*}
        \gamma \frac{d^{2} \phi_{i}(p_{i})}{dp_{i}^{2}} = - 4 \gamma \alpha.
    \end{equation*}
    If $\gamma > 0$, the second derivative is negative for all $p_{i} \in [0, 1]$, and this completes the proof that $\gamma \Phi(\B{p})$ is a concave function when $\gamma$ is positive (and vice versa).
\end{proof}
Combining Lemma \ref{lemma:app-component-lemma}, Lemma \ref{lemma:sum-penalty-enforcing} and Lemma \ref{lemma:sum-penalty-convex}, one can show the following theorem.
\begin{theorem}
    \label{theorem:app-gamma-annealing-limit}
    Under the assumption that the objective function $\hat{l}(\B{p};C)$ is bounded within the domain $[0, 1]^{N}$, as $\gamma \to +\infty$, the soft solutions $\B{p}^{\ast} \in \mathrm{argmin}_{\B{p}} \hat{r}(\B{p}; C, \B{\lambda}, \gamma)$ converge to the original solutions $\B{x}^{\ast} \in \mathrm{argmin}_{\B{x}} l(\B{x}; C, \B{\lambda})$. In addition, as $\gamma \to -\infty$, the loss function $\hat{r}(\B{p}; C, \B{\lambda}, \gamma)$ becomes convex, and the soft solution $\nicefrac{\B{1}_{N}}{2} = \mathrm{argmin}_{\B{p}} \hat{r}(\B{p}, C, \B{\lambda}, \gamma)$ is unique.
\end{theorem}

\begin{proof}
As $\lambda \to +\infty$, the penalty term $\Phi(\B{p})$ dominates the loss function $\hat{r}(\B{p}; C, \B{\lambda}, \gamma)$. According to Lemma \ref{lemma:sum-penalty-enforcing}, this penalty term forces the optimal solution $\B{p}^{\ast}$ to a binary vector whose components, for all $i \in [N]$ $p_{i}^{\ast}$ that are either $0$ or $1$ because any non-binary value results in an infinitely large penalty. 
This effectively restricts the feasible region to the vertices of the unit hypercube, which correspond to the binary vector in $\{0, 1\}^{N}$.
Thus, as $\lambda \to \infty$, the solutions to the relaxed problem converge to those of the original problem.
As $\lambda \to - \infty$, the penalty term $\Phi(\B{p})$ also dominates the loss function $\hat{r}(\B{p}; C, \B{\lambda}, \gamma)$ and the $\hat{r}(\B{p}; C, \B{\lambda})$ convex function from Lemma \ref{lemma:sum-penalty-convex}.
According to Lemma \ref{lemma:sum-penalty-enforcing}, this penalty term forces the optimal solution $\B{p}^{\ast}=\B{1}_{N}/2$. 
\end{proof}
The theorem holds for the cross entropy penalty given by
\begin{equation}
    \Phi(\B{p}) = \sum_{i=1}^{N} \left(p_{i} \log (p_{i}) + (1-p_{i}) \log (1-p_{i})  \right)
\end{equation}
in the UL-based solver using data or history \citep{sun2022annealed, sanokowski2024variational} because $\Phi(\B{p})$ can similarly satisfy Lemma \ref{lemma:app-component-lemma}, Lemma \ref{lemma:sum-penalty-enforcing} and Lemma \ref{lemma:sum-penalty-convex}.
\begin{corollary}
    Theorem \ref{theorem:app-gamma-annealing-limit} holds for the following penalty term:
    \begin{equation}
    \Phi(\B{p}) = \sum_{i=1}^{N} \left(p_{i} \log (p_{i}) + (1-p_{i}) \log (1-p_{i})  \right).
\end{equation}
\end{corollary}

\section{Generalization of CRA}\label{sec:generalization-cra}

\subsection{Generalization for to Potts variable optimization}\label{subsec:app-generalize-potts}
This section generalize the penalty term $\Phi(\B{\theta}; C)$ introduced for binary variables to $K$-Potts variables.
$K$-Potts variable is the Kronecker delta $\delta(x_{i}, x_{j})$ which equas one whenever $x_{i} = x_{j}$ and zero otherwise and a decision variable takes on $K$ different values, $\forall i \in [N],~x_{i} = 1, \ldots, K$.
For example, graph $K$-coloring problems can be expressed as 
\begin{equation}
    f(\B{x}; G(V, E)) = - \sum_{(i, j) \in E} \delta(x_{i}, x_{j})
\end{equation}
For Potts variables, the output of the GNN is $\B{h}_{\B{\theta}}^{L}=P(\B{\theta}; C) \in \mathbb{R}^{N \times K}$ and the penalty term can be generalized as follows. 
\begin{equation}
    \Phi(\B{\theta}; C) = \sum_{i=1}^{N} \left(1-\sum_{k=1}^{K} \left(K P_{i, k}(\B{\theta}; C)-1 \right)^{\alpha}   \right).
\end{equation}

\section{Additional implementation details}\label{sec:app-additional-implementation}

\subsection{Architecture of GNNs}\label{subsec:app-architecture-gnns}
We describe the details of the GNN architectures used in all numerical experiments.
The first convolutional layer takes $H_{0}$-dimensional node embedding vectors, $\B{h}_{\B{\theta}}^{0}$ for each node, as input, yielding $H_{1}$-dimensional feature vectors $\B{h}_{\B{\theta}}^{1}$. 
Then, the ReLU function is applied as a component-wise nonlinear transformation. 
The second convolutional layer takes the $H_{1}$-dimensional feature vectors, $\B{h}^{1}_{\B{\theta}}$, as input, producing a $H^{(2)}$-dimensional vector $\B{h}^{2}_{\B{\theta}}$. 
Finally, a sigmoid function is applied to the $H^{(2)}$-dimensional vector $\B{h}^{2}_{\B{\theta}}$, and the output is the soft solution $\B{p}_{\B{\theta}} \in [0, 1]^{N}$. 
As in \citep{schuetz2022combinatorial}, we set $H_{0}=\mathrm{int}(N^{0.8})$ or , $H_{1}=\mathrm{int}(N^{0.8}/2)$ and $H^{2}=1$ for both \texttt{GCN} and \texttt{GraphSAGE}
.

\subsection{Training settings}\label{subsec:app-learning-setting}
For all numerical experiments, we use the AdamW \citep{kingma2014adam} optimizer with a learning rate of $\eta =10^{-4}$ and weight decay of $10^{-2}$, and we train the GNNs for up to $10^{5}$ epochs with early stopping set to the absolute tolerance $10^{-5}$ and patience $10^{3}$. 
As discussed in \citet{schuetz2022combinatorial}, the GNNs are initialized with five different random seeds for a single instance because the results are dependent on the initial values of the trainable parameters; thus selecting the best solution. 

\section{Experiment details}\label{sec:experiment-details}
\subsection{Benchmark problems}\label{sec:theoretical-background-co-problem}
\paragraph{MIS problems}
There are some theoretical results for MIS problems on RRGs with the node degree set to $d$, where each node is connected to exactly $d$ other nodes. 
The MIS problem is a fundamental NP-hard problem \citep{karp2010reducibility} defined as follows.
Given an undirected graph $G(V, E)$,  an independent set (IS) is a subset of nodes $\mac{I} \in V$ where any two nodes in the set are not adjacent.
The MIS problem attempts to find the largest IS, which is denoted $\mac{I}^{\ast}$. 
In this study, $\rho$ denotes the IS density, where $\rho = |\mac{I}|/|V|$.
To formulate the problem, a binary variable $x_{i}$ is assigned to each node $i \in V$. Then the MIS problem is formulated as follows: 
\begin{equation}
    f(\B{x}; G, \lambda) = - \sum_{i \in V} x_{i} + \lambda \sum_{(i,j) \in E} x_{i} x_{j},
\end{equation}
where the first term attempts to maximize the number of nodes assigned $1$, and the second term penalizes the adjacent nodes marked $1$ according to the penalty parameter $\lambda$.
In our numerical experiments, we set $\lambda=2$, following \citet{schuetz2022combinatorial}, no violation is observed as in \citep{schuetz2022combinatorial}. 
First, for every $d$, a specific value $\rho_{d}^{\ast}$, which is dependent on only the degree $d$, exists such that the independent set density $|\mac{I}^{\ast}|/|V|$ converges to $\rho_{d}^{\ast}$ with a high probability as $N$ approaches infinity \citep{bayati2010combinatorial}. 
Second, a statistical mechanical analysis provides the typical MIS density $\rho^{\mathrm{Theory}}_{d}$, as shown in Fig.~\ref{fig:mis-degree-result}, 
and we clarify that for $d > 16$, the solution space of $\mac{I}$ undergoes a clustering transition, which is associated with hardness in sampling \citep{barbier2013hard} because the clustering is likely to create relevant barriers that affect any algorithm searching for the MIS $\mac{I}^{\ast}$. 
Finally, the hardness is supported by analytical results in a large $d$ limit, which indicates that, while the maximum independent set density is known to have density $\rho^{\ast}_{d \to \infty} = 2 \log(d)/d$, to the best of our knowledge, there is no known algorithm that can find an independent set density exceeding $\rho^{\mathrm{alg}}_{d \to \infty}=\log(d)/d$ \citep{coja2015independent}.

\paragraph{MaxCut problems}\label{subsec:theory-maxcut}
The MaxCut problem is also a fundamental NP-hard problem \citep{karp2010reducibility} with practical application in machine scheduling \citep{alidaee19940}, image recognition \citep{neven2008image} and electronic circuit layout design \citep{deza1994applications}.
The MaxCut problem is also a fundamental NP-hard problem \citep{karp2010reducibility}
Given an graph $G=(V, E)$, a cut set $\mac{C} \in E$ is defined as a subset of the edge set between the node sets dividing $(V_{1}, V_{2} \mid V_{1} \cup V_{2} = V,~V_{1} \cap V_{2} = \emptyset)$. 
The MaxCut problems aim to find the maximum cut set, denoted $\mac{C}^{\ast}$. 
Here, the cut ratio is defined as $\nu = |\mathcal{C}|/|\mathcal{V}|$, where $|\mathcal{C}|$ is the cardinality of the cut set.
To formulate this problem, each node is assigned a binary variable, where $x_{i}=1$ indicates that node $i$ belongs to $V_{1}$, and $x_{i}=0$ indicates that the node belongs to $V_{2}$. Here, $x_{i} + x_{j} - 2 x_{i} x_{j} = 1$ holds if the edge $(i, j) \in \mac{C}$. As a result, we obtain the following:
\begin{equation}
\label{eq:maxcut-qubo}
    f(\B{x}; G) = \sum_{i<j} A_{ij} (2x_{i}x_{j} - x_{i}-x_{j}).
\end{equation}
This study has also focused on the MaxCut problems on $d$-RRGs, for which several theoretical results have been established. Specifically, for each $d$, the maximum cut ratio is given by $\nu_{d}^{\ast} \approx d/4+P_{\ast} \sqrt{d/4} + \mac{O}(\sqrt{d})$, where $P_{\ast} = 0.7632\ldots$ with a high probability as $N$ approaches infinity \citep{parisi1980sequence, dembo2017extremal}. 
Thus, we take $\nu_{d}^{\mathrm{UB}} = d/4+P_{\ast} \sqrt{d/4}$ as an upper bound for the maximum cut ratio in the large $n$ limit.

\paragraph{DBM problems}
Here, the topologies are taken from the Cora citation network \citep{sen2008collective}, where each node has $1{,}433$ bag-of-words features, and each edge represents likelihood, as predicted by a machine learning model.
    \citet{mandi2022decision} focused on disjoint topologies within the given topology, and they created $27$ distinct instances with varying properties. 
    Each instance comprises $100$ nodes representing scientific publications, divided into two groups of $50$ nodes $N_{1}$ and $N_{2}$. 
    The optimization task is to find the maximum matching, where diversity constraints ensure connections among papers in the same field and between papers of different fields. 
    This is formulated using a penalty method as follows.
    \begin{multline}
    \label{eq:matching-cost-function}
        l(\B{x}; C, M, \B{\lambda}) = - \textstyle\sum_{ij} C_{ij} x_{ij} + \lambda_{1} \textstyle\sum_{i}\mathrm{ReLU}\Big(\textstyle\sum_{j} x_{ij}-1 \Big) + \lambda_{2} \textstyle\sum_{j}\mathrm{ReLU}\Big(\sum_{i} x_{ij} - 1\Big) \\
        + \lambda_{3} \mathrm{ReLU}\Big(p\textstyle\sum_{ij} x_{ij} - \textstyle\sum_{ij} M_{ij} x_{ij} \Big) + \lambda_{4} \mathrm{ReLU}\Big(q\sum_{ij} x_{ij} - \sum_{ij} (1-M_{ij}) x_{ij} \Big),
    \end{multline}
    where $C \in \mab{R}^{N_{1} \times N_{2}}$ represents the likelihood of a link between each pair of nodes, an indicator $M_{ij}$ is set to $0$ if article $i$ and $j$ share the same subject field ($1$ otherwise) $\forall i \in N_{1}$, and $j \in N_{2}$. The parameters $p, q \in [0, 1]$ represent the probability of pairs sharing their field and of unrelated pairs, respectively. 
    As in \citet{mandi2022decision}, we explore two variations of this problem, with $p=q=$ being $25$\% and $5$\%, respectively, and these variations are referred to as Matching-1 and Matching-2, respectively. In this experiment, we set $\lambda_{1}=\lambda_{2}=10$ and $\lambda_{3}=\lambda_{4}=25$.

\subsection{GNNs}\label{subsec:graph-neural-network}
A GNN \citep{gilmer2017neural, scarselli2008graph} is a specialized neural network for representation learning of graph-structured data. 
GNNs learn a vectorial representation of each node in two steps, i.e., the aggregate and combine steps. The aggregate step employs a permutation-invariant function to generate an aggregated node feature, and in the combine step, the aggregated node feature is passed through a trainable layer to generate a node embedding, known as "message passing" or the "readout phase."
Formally, for a given graph $G = (V, E)$, where each node feature $\B{h}^{0}_{v} \in \mab{R}^{H_{0}}$ is attached to each node $v \in V$, the GNN updates the following two steps iteratively. First, the aggregate step at each $l$-th layer is defined as follows:
\begin{equation}
    \B{a}_{v}^{l} = \mathrm{Aggregate}_{\B{\theta}}^{l}\left(\{h_{u}^{l-1}, \forall u \in \mac{N}_{v} \}\right),
\end{equation}
where the neighborhood of $v \in V$ is denoted $\mac{N}_{v} = \{u \in V \mid (v, u) \in E \}$, $\B{h}_{u}^{l-1}$ is the node feature of the neighborhood, and $\B{a}_{v}^{l}$ is the aggregated node feature of the neighborhood. Then, the combined step at each $l$-th layer is defined as follows:
\begin{equation}
    \B{h}_{v}^{l} = \mathrm{Combine}^{l}_{\B{\theta}}(\B{h}_{v}^{l-1}, \B{a}_{v}^{l}),
\end{equation}
where $\B{h}_{v}^{l} \in \mab{R}^{H_{l}}$ denotes the node representation at the $l$-th layer. 
Here, the hyperparameters for the total number of layers $L$ and the intermediate vector dimension $N^{l}$ are determined empirically. 
Although numerous implementations of GNN architectures have been proposed to date, the most basic and widely used architecture is the GCN \citep{scarselli2008graph}, which is given as follows:
\begin{equation}
    \B{h}_{v}^{l} = \sigma \left(W^{l} \sum_{u \in \mac{N}(v)} \frac{\B{h}_{u}^{l-1}}{|\mac{N}(v)|} + B^{l} \B{h}_{v}^{l-1}  \right),
\end{equation}
where $W^{l}$ and $B^{l}$ are trainable parameters, $|\mac{N}(v)|$ serves as a normalization factor, and $\sigma: \mab{R}^{H_{l}} \to \mab{R}^{H_{l}}$ is some component-wise nonlinear activation function, e.g., the sigmoid or ReLU function.

\section{Additional experiments}\label{sec:additional-experimet}

\subsection{Numerical validation of practical issues presented in Section \ref{sec:how-can-it-fail}}
\label{subsec:app-how-can-it-fail}
In this sectioin, we will examine the issue (I) with continuous relaxations and the issue (II), the difficulties of optimization, as pointed out by previous studies \citep{wang2023unsupervised, angelini2023modern}, in the NP-hard problems of MIS and the MaxCut problem.
Therefere, we conducted numerical experiments using the PI-GNN solver for MIS and MaxCut problems on RRGs with higher degrees. 
To ensure the experimental impartiality, we adhered to the original settings of the PI-GNN solver \citep{schuetz2022graph}. Refer to Section \ref{sec:experiment-details} for the detailed experimental settings.
Fig.~\ref{fig:fail-example} (top) shows the solutions obtained by the PI-GNN solver as a function of the degree $d$ for the MIS and MaxCut problems with varying system sizes $N$. 
These results indicate that finding independent and cut sets becomes unfeasible as the RRG becomes denser. 
In addition, to clarify the reasons for these failures, we analyzed the dynamics of the cost function for MIS problems with $N=10{,}000$, with a specific focus on a graph with degrees $d=5$ and $d=20$, as depicted in Fig. ~\ref{fig:fail-example} (bottom).
For the $d=5$ case, the cost function goes over the plateau of $\hat{l}(\B{\theta}; G, \B{\lambda})=0$ with $\B{p}_{\theta}(G) = \B{0}_{N}$, as investigated in the histogram, eventually yielding a solution comparable to those presented by \citet{schuetz2022combinatorial}. 
Conversely, in the $d=20$ case, the cost function remains stagnant on the plateau of $\hat{l}(\B{\theta}; G, \B{\lambda})=0$ with $\B{p}_{\theta}(G)=\B{0}_{N}$, thereby failing to find any independent nodes. 
Interpreting this phenomenon, we hypothesize that the representation capacity of the GNN is sufficiently large, leading us to consider the optimization of $\hat{L}_{\mathrm{MIS}}(\B{\theta}; G, \lambda)$ and $\hat{L}_{\mathrm{MaxCut}}(\B{\theta}; G)$ as a variational optimization problem relative to $\B{p}_{\theta}$. 
In this case, $\B{p}_{\theta}^{\ast} = \B{0}_{N}$ satisfies the first-order variational optimality conditions $\delta \hat{l}_{\mathrm{MIS}}/\delta \B{p}_{\theta}|_{\B{p}_{\theta}=\B{p}^{\ast}} = \delta \hat{l}_{\mathrm{MaxCut}}/\delta \B{p}_{\theta}|_{\B{p_{\theta}}=\B{p}^{\ast}} = \B{0}_{N}$, which implies a potential reason for absorption into the plateau. 
However, this does not reveal the conditions for the convergence to the fixed point $\B{p}^{\ast}$ during the early learning stage or the condition to escape from the fixed point $\B{p}^{\ast}$. 
Thus, an extensive theoretical evaluation through stability analysis remains an important topic for future work.

\begin{figure}[tb]
    \centering
    \includegraphics[width=0.8\linewidth]{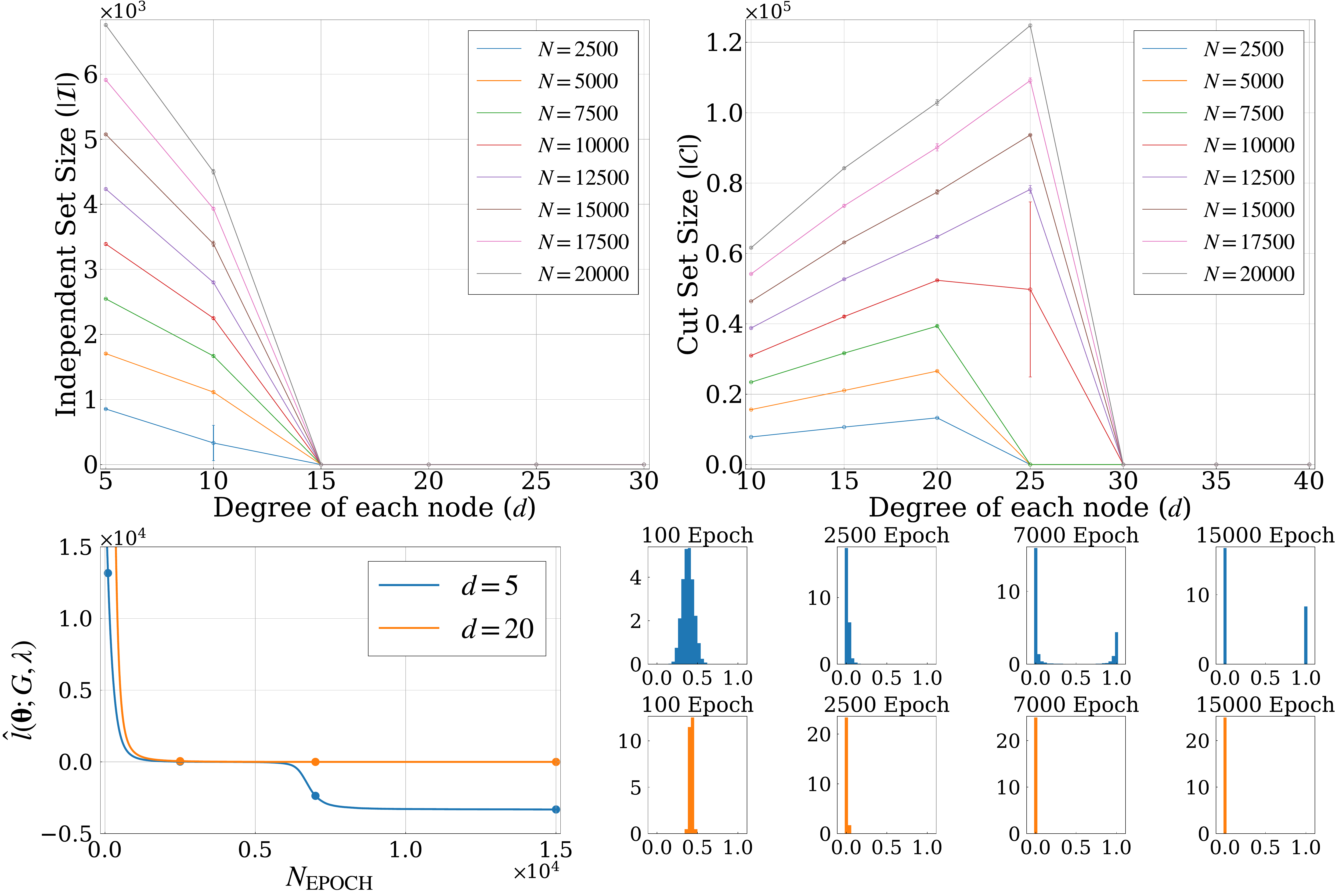}
    \caption{The top graph shows the independent set density for MIS problems (left) and the cut ratio for MaxCut problems (right) as a function of degree $d$ using the PI-GNN solver with varying system size $N$. 
    Each data point represents the average result of five different graph instances, with the error bars indicating the standard deviation of those results. The bottom graph shows the cost as a function of the number of parameter updates $N_{\mathrm{EPOCH}}$, for $N=10000$ MIS problems on $5$-RRG and $20$-RRG.
    The histogram represents the relaxed vector distribution with varying numbers of parameter updates $N_{\mathrm{EPOCH}}$. 
    Each point in the bottom-left plot is linked to the corresponding bottom-right histogram.}
    \label{fig:fail-example}
\end{figure}

In summary, UL-based solver, minimizing $\B{\theta}$ can be challenging and unstable.
In particular, the PI-GNN solver, which is one of the UL-based solvers employing GNNs, fails to optimize $\B{\theta}$ due to a local solution in complex CO problems on relatively dense graphs where the performance of greedy algorithms worsens.
This issues can be potential bottleneck for more practical and relatively dense problems, making it challenging to employ the PI-GNN solver confidently.

\subsection{Additional results of MIS}\label{subsec:app-additional-results-mis}
\begin{table}[tb]
\label{table:mis-satlib-er}
\centering
\caption{ApR and runtime are evaluated on three benchmarks provided by DIMES \citep{qiu2022dimes}. The ApR is assessed relative to the results obtained by KaMIS. Runtime is reported as the total clock time, denoted in seconds (s), minutes (m), or hours (h). The runtime and solution quality are sourced from iSCO \citep{sun2023revisiting}. 
The baselines include solvers from the Operations Research (OR) community, as well as data-driven approaches utilizing Reinforcement Learning (RL), Supervised Learning (SL) combined with Tree Search (TS), Greedy decoding (G), or sampling (S). Methods that fail to produce results within 10 times the time limit of DIMES are marked as N/A.}
\begin{tabular}{cc|cc|cc}
     \toprule
     \multirow{2}{*}{Method} & \multirow{2}{*}{Type}  & 
     \multicolumn{2}{c|}{ER-[700-800]} & \multicolumn{2}{c}{ER-[9000-11000]}\\
     & & ApR & Time & ApR & Time \\
     \midrule
     KaMIS & OR & 1.000  & 52.13m & 1.000  & 7.6h\\
     Gurobi & OR & 0.922  & 50.00m & N/A & N/A \\
     \midrule
     \multirow{2}{*}{Intel}  & SL+TS  & 0.865  & 20.00m & N/A & N/A \\
    & SL+G  & 0.777  &  6.06m & 0.746  & 5.02m \\
     DGL & SL+TS  & 0.830  & 22.71m & N/A & N/A \\
     LwD & RL+S  & 0.918  & 6.33m & 0.907  & 7.56m \\
     \multirow{2}{*}{DIMES} & RL+G  & 0.852 & 6.12m & 0.841 & 5.21m\\
     & RL+S  & 0.937 & 12.01m & 0.873 & 12.51m\\
     \midrule
     \multirow{2}{*}{iSCO} & fewer steps & 0.998 & 1.38m & 0.990 & 9.38m \\
     & more steps  & 1.006 & 5.56m & 1.008 & 1.25h\\
     \midrule
     CRA & UL-based & 0.928  & 47.30m & 0.963 & 1.03h \\
     \bottomrule
\end{tabular}
\end{table}
We evaluate our method using the MIS benchmark dataset from recent studies \citep{NEURIPS2023_f9ad87c1, qiu2022dimes}, which includes graphs from SATLIB \citep{hoos2000satlib} and Erd\H{o}s–R\'{e}nyi graphs (ERGs) of varying sizes. Following \citet{sun2023revisiting}, our test set consists of $500$ SATLIB graphs, each containing between $403$ and $449$ clauses with up to $1{,}347$ nodes and $5{,}978$ edges, $128$ ERGs with $700$ to $800$ nodes each, and $16$ ERGs with $9{,}000$ to $11{,}000$ nodes each.
We conducted numerical experiments on PQQA using four different configurations: parallel runs with $S=100$ or $S=1000$ and shorter steps (3000 steps) or longer steps (30000 steps), similar to the approach in iSCO \citep{sun2023revisiting}. 
Table \ref{table:mis-satlib-er} presents the solution quality and runtime results. 
The results show that CRA, which optimizes the relaxed variables as an optimization of GNN parameters, takes extra time for smaller ER-[700-800] instances due to the smaller number of decision variables. However, for larger instances, CRA achieves results comparable to iSCO. Although limited space makes it difficult to present other benchmark results employed by iSCO, such as MaxCut and MaxClique, numerical experiments on these benchmarks also show that CRA is less effective for small problems. However, for larger problems, the results are comparable to or slightly inferior to those of iSCO. 

We also investigated the relationship between the order of the graph and the solving time of the solver, and the results are shown in Table \ref{table:degree-dependence-mis-RRG} and \ref{table:degree-dependence-mis-ERG}.
The results demonstrate that the runtime remains nearly constant as graph order and density increase, indicating effective scalability with denser graphs.

\begin{table}[tp]
\centering
\label{table:degree-dependence-mis-RRG}
\caption{ApR of the MIS problem on $\mathrm{RRG}(N, d)$.  All the results are averaged based on 5 $\mathrm{RRG}$s with different random seeds.}
\begin{tabular}{c|cccc}
\toprule
Problem & ApR (CRA) & ApR (PI) & Time (CRA) & Time (PI) \\
\midrule
$\mathrm{RRG}(1{,}000, 10)$ & 0.95 & 0.78 & 108 (s) & 98 (s) \\
$\mathrm{RRG}(1{,}000, 20)$ & 0.95 & 0.56 & 103 (s) & 92 (s) \\
$\mathrm{RRG}(1{,}000, 30)$ & 0.94 & 0.00 & 102 (s) & 88 (s) \\
$\mathrm{RRG}(1{,}000, 40)$ & 0.93 & 0.00 & 101 (s) & 82 (s) \\
$\mathrm{RRG}(1{,}000, 50)$ & 0.92 & 0.00 & 102 (s) & 82 (s) \\
$\mathrm{RRG}(1{,}000, 60)$ & 0.91 & 0.00 & 101 (s) & 91 (s) \\
$\mathrm{RRG}(1{,}000, 70)$ & 0.91 & 0.00 & 101 (s) & 86 (s) \\
$\mathrm{RRG}(1{,}000, 80)$ & 0.91 & 0.00 & 102 (s) & 93 (s) \\
$\mathrm{RRG}(5{,}000, 10)$ & 0.93 & 0.77 & 436 (s) & 287 (s) \\
$\mathrm{RRG}(5{,}000, 20)$ & 0.95 & 0.74 & 413 (s) & 280 (s) \\
$\mathrm{RRG}(5{,}000, 30)$ & 0.95 & 0.00 & 419 (s) & 283 (s) \\
$\mathrm{RRG}(5{,}000, 40)$ & 0.94 & 0.00 & 429 (s) & 293 (s) \\
$\mathrm{RRG}(5{,}000, 50)$ & 0.94 & 0.00 & 418 (s) & 324 (s) \\
$\mathrm{RRG}(5{,}000, 60)$ & 0.93 & 0.00 & 321 (s) & 302 (s) \\
$\mathrm{RRG}(5{,}000, 70)$ & 0.92 & 0.00 & 321 (s) & 325 (s) \\
$\mathrm{RRG}(5{,}000, 80)$ & 0.92 & 0.00 & 330 (s) & 305 (s) \\
\bottomrule
\end{tabular}
\end{table}

\begin{table}[tb]
\centering
\label{table:degree-dependence-mis-ERG}
\caption{The ApR of the MIS problem on $\mathrm{ERG}(N, d)$ is evaluated against the results of KaMIS. Due to time limitations, the maximum running time for KaMIS was constrained. The results below present the average ApRs and runtimes across five different random seeds.}
\begin{tabular}{lcccccc}
\toprule
Problem & CRA (ApR) & PI (ApR) & Time (CRA) & Time (PI) & Time (KaMIS) \\
\midrule
$\mathrm{ERG}(1{,}000, 0.05)$ & 0.97 & 0.01 & 103 (s) & 98 (s) & 100 (s) \\
$\mathrm{ERG}(1{,}000, 0.10)$ & 0.95 & 0.00 & 100 (s) & 98 (s) & 210 (s) \\
$\mathrm{ERG}(1{,}000, 0.15)$ & 0.94 & 0.00 & 100 (s) & 92 (s) & 315 (s) \\
$\mathrm{ERG}(1{,}000, 0.20)$ & 0.91 & 0.00 & 99 (s) & 88 (s) & 557 (s) \\
$\mathrm{ERG}(1{,}000, 0.25)$ & 0.93 & 0.00 & 98 (s) & 82 (s) & 733 (s) \\
$\mathrm{ERG}(1{,}000, 0.30)$ & 0.90 & 0.00 & 98 (s) & 82 (s) & 1000 (s) \\
$\mathrm{ERG}(1{,}000, 0.35)$ & 0.92 & 0.00 & 99 (s) & 91 (s) & 1000 (s) \\
$\mathrm{ERG}(1{,}000, 0.40)$ & 0.91 & 0.00 & 97 (s) & 86 (s) & 1000 (s) \\
\bottomrule
\end{tabular}
\end{table}

\subsection{Additional results of Gset}\label{subsec:app-additional-results}
We conducted experiments across the additional GSET collection to further validate that including CRA enhances PI-GNN results beyond previously achievable in Table \ref{tab:app-maxcut-gset}.

\begin{table}[tb]
    \caption{ApR for MaxCut on Gset}
    \label{tab:app-maxcut-gset}
        \begin{center}
        \begin{small}
        \begin{sc}
            \begin{tabular}{lccccccc}
                \toprule
                GRAPH & (NODES, EDGES) & GREEDY & SDP & RUN-CSP & PI-GNN & CRA  \\
                \midrule
                G1 & ($800$, $19{,}176$) & $0.942$ & $0.970$ &$0.979$ & $0.978$ & $\mathbf{1.000}$  \\
                G2 & ($800$, $19{,}176$) & $0.951$ & $0.970$ &$0.981$ & $0.976$ & $\mathbf{0.998}$  \\
                G3 & ($800$, $19{,}176$) & $0.945$ & $0.972$ &$0.982$ & $0.972$ & $\mathbf{1.000}$  \\
                G4 & ($800$, $19{,}176$) & $0.949$ & $0.971$ &$0.980$ & $0.978$ & $\mathbf{0.999}$  \\
                G5 & ($800$, $19{,}176$) & $0.949$ & $0.970$ &$0.980$ & $0.978$ & $\mathbf{1.000}$  \\
                G14 & ($800$, $4{,}694$) & $0.946$ & $0.952$ &$0.956$ & $0.988$ & $\mathbf{0.994}$  \\
                G15 & ($800$, $4{,}661$) & $0.939$ & $0.958$ &$0.952$ & $0.980$ & $\mathbf{0.992}$  \\
                G16 & ($800$, $4{,}672$) & $0.948$ & $0.958$ &$0.953$ & $0.965$ & $\mathbf{0.990}$  \\
                G17 & ($800$, $4{,}667$) & $0.946$ & $-$     &$0.956$ & $0.967$ & $\mathbf{0.990}$  \\
                G22 & ($2{,}000$, $19{,}990$) & $0.923$ & $-$ &$0.972$ & $0.987$ & $\mathbf{0.998}$  \\
                G23 & ($2{,}000$, $19{,}990$) & $0.927$ & $-$ &$0.973$ & $0.968$ & $\mathbf{0.997}$  \\
                G24 & ($2{,}000$, $19{,}990$) & $0.927$ & $-$ &$0.973$ & $0.959$ & $\mathbf{0.998}$  \\
                G25 & ($2{,}000$, $19{,}990$) & $0.929$ & $-$ &$0.974$ & $0.974$ & $\mathbf{0.998}$  \\
                G26 & ($2{,}000$, $19{,}990$) & $0.924$ & $-$ &$0.974$ & $0.965$ & $\mathbf{0.998}$  \\
                G35 & ($2{,}000$, $11{,}778$) & $0.942$ & $-$ &$0.953$ & $0.968$ & $\mathbf{0.990}$  \\
                G36 & ($2{,}000$, $11{,}766$) & $0.942$ & $-$ &$0.953$ & $0.966$ & $\mathbf{0.991}$  \\
                G37 & ($2{,}000$, $11{,}785$) & $0.946$ & $-$ &$0.950$ & $0.966$  & $\mathbf{0.997}$  \\
                G38 & ($2{,}000$, $11{,}779$) & $0.943$ & $-$ & $0.949$ & $0.966$ & $\mathbf{0.991}$   \\
                G43 & ($1{,}000$, $9{,}990$) & $0.928$ & $0.968$ & $0.976$ &  $0.966$ & $\mathbf{0.995}$  \\
                G44 & ($1{,}000$, $9{,}990$) & $0.920$ & $0.955$ & $0.978$ & $0.968$  & $\mathbf{0.998}$  \\
                G45 & ($1{,}000$, $9{,}990$) &$0.930$ & $0.950$ & $0.979$ & $0.961$ & $\mathbf{0.998}$  \\
                G46 & ($1{,}000$, $9{,}990$) & $0.930$ & $0.960$ &$0.976$ & $0.974$ & $\mathbf{0.998}$  \\
                G47 & ($1{,}000$, $9{,}990$) & $0.931$ & $0.956$ &$0.976$ & $0.972$  & $\mathbf{0.997}$  \\
                G48 & ($3{,}000$, $6{,}000$) & $\mathbf{1.000}$ & $\mathbf{1.000}$ &$\mathbf{1.000}$ & $0.912$  & $\mathbf{1.000}$  \\
                G49 & ($3{,}000$, $6{,}000$) & $\mathbf{1.000}$ & $\mathbf{1.000}$ &$\mathbf{1.000}$ & $0.986$  & $\mathbf{1.000}$  \\
                G50 & ($3{,}000$, $6{,}000$) & $\mathbf{1.000}$ & $\mathbf{1.000}$ &$0.999$ & $0.990$  & $\mathbf{1.000}$  \\
                G51 & ($1{,}000$, $5{,}909$) & $0.949$ & $0.960$ &$0.954$ & $0.964$  & $\mathbf{0.991}$  \\
                G52 & ($1{,}000$, $5{,}916$) & $0.944$ & $0.957$ &$0.955$ & $0.961$  & $\mathbf{0.990}$  \\
                G53 & ($1{,}000$, $5{,}914$) & $0.945$ & $0.956$ &$0.950$ & $0.966$  & $\mathbf{0.992}$  \\
                G54 & ($1{,}000$, $5{,}916$) & $0.900$ & $0.956$ &$0.952$ & $0.970$  & $\mathbf{0.988}$  \\
                G55 & ($5{,}000$, $12{,}498$) & $0.892$ & $-$ &$0.978$ & $0.983$  & $\mathbf{0.991}$  \\
                G58 & ($5{,}000$, $29{,}570$) & $0.945$ & $-$ &$0.980$ & $0.966$  & $\mathbf{0.989}$  \\
                G60 & ($7{,}000$, $17{,}148$) & $0.889$ & $-$ &$0.980$ & $0.945$  & $\mathbf{0.991}$  \\
                G63 & ($7{,}000$, $41{,}459$) & $0.948$ & $-$ &$0.947$ & $0.968$  & $\mathbf{0.989}$  \\
                G70 & ($10{,}000$, $9{,}999$) & $0.886$ & $-$ &$0.970$ & $0.982$  & $\mathbf{0.992}$ \\
                \bottomrule
            \end{tabular}
        \end{sc}
        \end{small}
        \end{center}
\end{table}

\subsection{Additional results of TSP}\label{subsection:additional-results-TSP}
We conducted additional experiments on several TSP problems from the TSPLIB dataset\footnote{\url{http://comopt.ifi.uni-heidelberg.de/software/TSPLIB95/}}, presenting results that illustrate the $\alpha$-dependency.

Experiments calculated the ApR as the ratio of the optimal value to the CRA result, with the ApR representing the average and standard deviation over 3 seeds. 
The ``--'' symbol in PI-GNN denotes cases where most variables are continuous values and where no solution satisfying the constraint was found within the maximum number of epochs. 
The same GNN and optimizer settings were used as in the main text experiments in Section \ref{sub-sec:config}.

Table \ref{table:additional-experiments-TSP} shows that the CRA approach yielded solutions with an ApR exceeding 0.9 across various instances. 
Notably, for the burma14 problem, our method identified the global optimal solution (3,323) multiple times.
However, the optimal value can vary based on the specific GNN architecture and problem structure, indicating that a more comprehensive ablation study could provide valuable insights in future work.

\begin{figure}[tb]
    \centering
    \includegraphics[width=\linewidth]{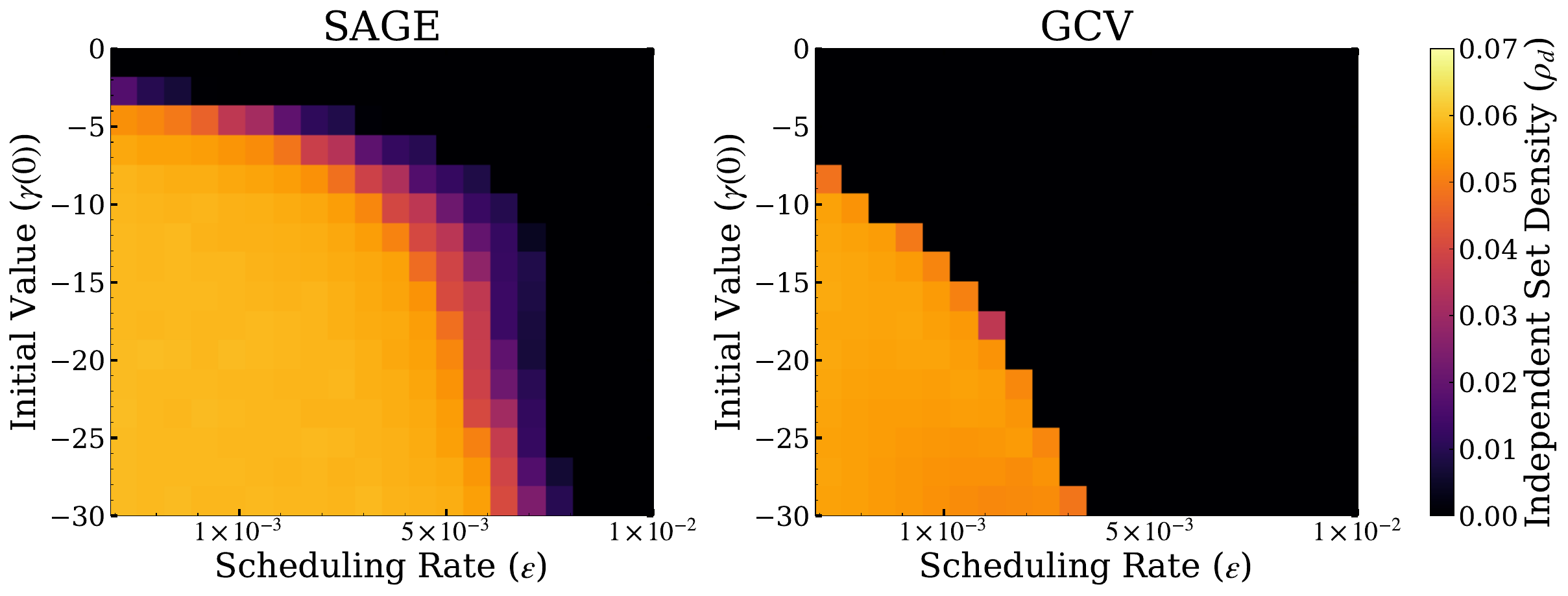}
    \caption{(Top) IS density of $N=10{,}000$ MIS problems on $100$-RRG as a function of initial scheduling $\gamma(0)$ and scheduling rate $\varepsilon$ values obtained by the CRA-PI-GNN solver using GraphSage (Left) and GCV (Right). 
    The color of the heat map represents the average IS over five different instances.}
    \label{fig:scheduling-ablation}
\end{figure}

\begin{figure}[tb]
    \centering
    \includegraphics[width=0.9\linewidth]{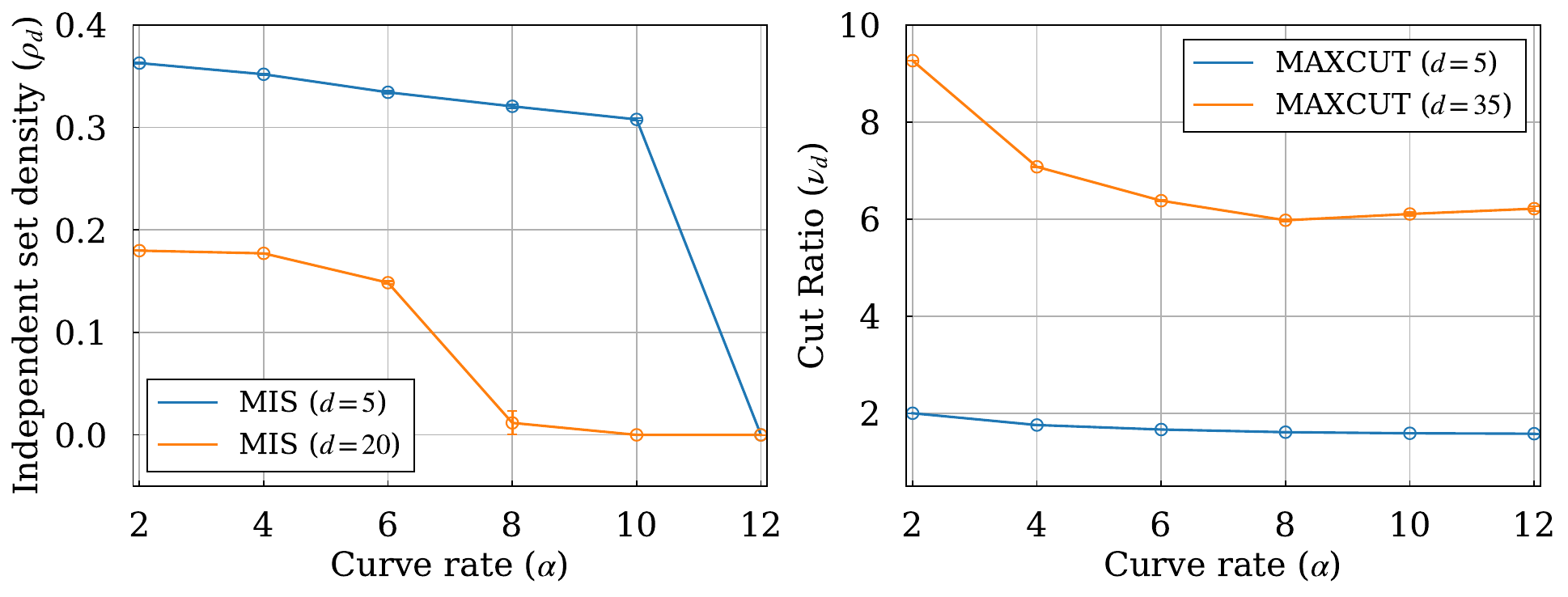}
    \caption{(Left) Independent set density as a function of curveture rate $\alpha$ in Eq.~\eqref{eq:penalty-cost}. (Right) Cut ratio as a function of curveture rate $\alpha$ in Eq.~\eqref{eq:penalty-cost}.
    Error bars represent the standard deviations of the results.}
    \label{fig:ablation_curverate}
\end{figure}

\begin{table}[tp]
\centering
\label{table:additional-experiments-TSP}
\caption{Comparison of ApR performance across different values of $p$ for various TSPLIB instances.}
\begin{tabular}{lcccc}
\toprule
& burma14 & ulysses22 & st70 & gr96 \\
\midrule
ApR ($p=2$) & $0.91 \pm 0.08$ & $0.89 \pm 0.03$ & $0.96 \pm 0.01$ & $0.81 \pm 0.05$ \\
ApR ($p=4$) & $0.98 \pm 0.10$ & $0.92 \pm 0.02$ & $0.85 \pm 0.03$ & $0.82 \pm 0.03$ \\
ApR ($p=6$) & $0.97 \pm 0.14$ & $0.88 \pm 0.07$ & $0.88 \pm 0.04$ & $0.90 \pm 0.05$ \\
ApR ($p=8$) & $0.99 \pm 0.06$ & $0.89 \pm 0.05$ & $0.80 \pm 0.02$ & $0.86 \pm 0.05$ \\
ApR (PI) & $0.736 \pm 1.21$ & -- & -- & -- \\
Optimum & 3,323 & 7,013 & 675 & 5,5209 \\
\bottomrule
\end{tabular}
\end{table}

\subsection{Ablation over initial scheduling value and scheduling rate}\label{subsec:ablation-study}
We conducted an ablation study focusing on the initial scheduling value $\gamma(0)$ and scheduling rate $\varepsilon$.
This numerical experiment was conducted under the configuration described in Section~\ref{sec:experiment} and \ref{sec:experiment-details}. Fig. ~\ref{fig:scheduling-ablation} shows the IS density of $N=10{,}000$ for MIS problems on a $100$-RRG as a function of the initial scheduling value $\gamma(0)$ and the scheduling rate $\varepsilon$ using the CRA-PI-GNN with both GraphSage and GCV.
As can be seen, smaller initial scheduling $\gamma(0)$ and scheduling rate $\varepsilon$ values typically yield better solutions. 
However, the convergence time increases progressively as the initial scheduling $\gamma(0)$ and scheduling rate $\varepsilon$ values become smaller.
In addition, GraphSage consistently produces better solutions even with relatively larger initial scheduling $\gamma(0)$ and scheduling rate $\varepsilon$ values, which implies that the GNN architecture influences both the solution quality and the effective regions of the initial scheduling $\gamma(0)$ and scheduling rate $\varepsilon$ values for the annealing process.

\subsection{Ablation over curve rate}\label{subsec:ablation-curve-rate}
Next, we investigated the effect of varying the curvature $\alpha$ in Eq.~\ref{eq:penalty-cost}.
Numerical experiments were performed on MIS problems with $10{,}000$ nodes and the degrees of $5$ and $20$, as well as MaxCut problems with $10{,}000$ nodes and the degrees of $5$ and $35$. 
The GraphSAGE architecture was employed, with other parameters set as Section~ \ref{sec:experiment} adn \ref{sec:experiment-details}.
As shown in Fig.~\ref{fig:ablation_curverate}, we observed that $\alpha=2$ consistently yielded the best scores across these problems.

\end{document}